\documentclass[journal]{IEEEtran}
\usepackage{amsmath,amsfonts,amsthm,amstext,amssymb,mathtools}
\usepackage{algorithm,algpseudocode}
\usepackage{array,makecell,soul,multirow,multicol} 
\usepackage{textcomp}
\usepackage{stfloats}
\usepackage{url}
\usepackage{verbatim}
\usepackage{graphicx}

\usepackage{subcaption}
\usepackage{epsfig,psfrag,subeqnarray}
\usepackage[colorlinks,linkcolor=blue, citecolor=blue]{hyperref}  
\usepackage{epstopdf,bm}
\usepackage{cancel}
\usepackage{booktabs}
\usepackage{color,colortbl}
\usepackage{setspace}
\usepackage{adjustbox}
\usepackage{pifont}
\usepackage{latexsym}
\usepackage{tikz}
\usetikzlibrary{bayesnet,calc,decorations.pathreplacing}
\usepackage{bbm}
\usepackage{orcidlink}
\usepackage{flushend}

\newtheorem{example}{Example}

\newtheorem{corollary}{Corollary}
\newtheorem{remark}{Remark}
\newtheorem{definition}{Definition}
\newtheorem{proposition}{Proposition}
\newtheorem{assumption}{Assumption}

\def\x{{\mathbf x}}
\def\y{{\mathbf y}}
\def\z{{\mathbf z}}
\def\f{{\mathbf{f}}}
\def\u{{\mathbf u}}
\def\m{{\mathbf m}}

\def\G{{\mathbb G}}

\def\cN{{\cal N}}

\def\bzeta{{\bm{\zeta}}}
\def\btheta{{{\bm{\theta}}}}

\definecolor{orange}{RGB}{200,0,100}

\newcommand{\cred}[1]{{\color{red}#1}}
\newcommand{\cblue}[1]{{\color{black}#1}}

\begin{document}

\title{Efficient Transformed Gaussian Process State-Space Models for Non-Stationary High-Dimensional Dynamical Systems}

\author{Zhidi Lin$^{\orcidlink{0000-0002-6673-511X}}$, 
Ying Li$^{\orcidlink{0009-0007-8100-6568}}$, 
Feng Yin$^{\orcidlink{0000-0001-5754-9246}}$,~\IEEEmembership{Senior Member,~IEEE},  
Juan Maro\~{n}as$^{\orcidlink{0000-0002-2939-7676}}$, and 
Alexandre H. Thi\'ery$^{\orcidlink{0000-0002-9542-509X}}$
\thanks{The work of Feng Yin was supported in part by the Shenzhen Science and Technology Program under Grant JCYJ20220530143806016 and by the NSFC under Grant No. 62271433. 
The work of Juan Maroñas was supported by the Spanish Government under the project PID2022-139856NB-I00.
The work of Alexandre H. Thiéry was supported by the Singapore Ministry of Education under grant MOE-T2EP20123-0010. (\textit{Corresponding authors: Feng Yin, Alexandre H. Thiéry})}
\thanks{
Zhidi Lin is with the Department of Statistics and Actuarial Science, University of Hong Kong, Hong Kong SAR, China  (Email: \href{mailto:zhidilin@hku.hk}{zhidilin@hku.hk}).
}
\thanks{
Ying Li and Alexandre H. Thiéry are with the Department of Statistics and Data Science, National University of Singapore, Singapore 117546  (Email: \href{mailto:ying-li@nus.edu.sg}{ying-li@nus.edu.sg}, \href{mailto:a.h.thiery@nus.edu.sg}{a.h.thiery@nus.edu.sg}).
}
\thanks{
Feng Yin is with the School of Artificial Intelligence, The Chinese University of Hong Kong, Shenzhen, Shenzhen 518172, China (Email: \href{mailto:yinfeng@cuhk.edu.cn}{yinfeng@cuhk.edu.cn}).
}
\thanks{
Juan Maro\~{n}as is with the Machine Learning Group, Universidad Aut\'{o}noma de Madrid, and with the Department of Quantitative Methods at CUNEF Universidad  (Email: \href{mailto:juan.maronas@cunef.edu}{juan.maronas@cunef.edu}). 
}
\thanks{This paper has supplementary downloadable material available at \url{http://ieeexplore.ieee.org}, provided by the author. The material includes additional method discussions and experimental details. This material is 199KB in size.}
}
\maketitle
\begin{abstract}
Gaussian process state-space models (GPSSMs) offer a principled framework for learning and inference in nonlinear dynamical systems with uncertainty quantification. However, existing GPSSMs are limited by the use of multiple independent stationary Gaussian processes (GPs), leading to prohibitive computational and parametric complexity in high-dimensional settings and restricted modeling capacity for non-stationary dynamics. 
To address these challenges, we propose an efficient transformed Gaussian process state-space model (ETGPSSM) for scalable and flexible modeling of high-dimensional, non-stationary dynamical systems. Specifically, our ETGPSSM integrates a single shared GP with input-dependent normalizing flows, yielding an expressive non-stationary implicit process prior that can capture complex transition dynamics while significantly reducing model complexity. 
For the inference of the implicit process, we develop a variational inference algorithm that jointly approximates the posterior over the underlying GP and the neural network parameters defining the normalizing flows. To avoid explicit variational parameterization of the latent states, we further incorporate the ensemble Kalman filter (EnKF) into the variational framework, enabling accurate and efficient state estimation.
Extensive empirical evaluations on synthetic and real-world datasets demonstrate the superior performance of our ETGPSSM in system dynamics learning, high-dimensional state estimation, and time-series forecasting, outperforming existing GPSSMs and neural network-based SSMs in terms of computational efficiency and accuracy.
\end{abstract}


\newpage
\section{Introduction}
\label{sec:intro} 

\IEEEPARstart{S}{tate-space} models (SSMs) provide a flexible framework for modeling dynamical systems, where latent states evolve based on internal dynamics and external inputs \cite{sarkka2013bayesian}. When system dynamics are fully known, Bayesian filtering methods—such as the Kalman filter (KF), extended Kalman filter (EKF), ensemble Kalman filter (EnKF), and particle filter (PF)—enable sequential state estimation by combining prior knowledge with incoming observations, yielding reliable predictions under uncertainty. This capability makes SSMs widely applicable in robotics, finance, and signal processing, where modeling temporal dependencies is crucial \cite{kullberg2021online, khan2008distributing, tobar2015unsupervised}.

However, in complex real-world systems, such as climate modeling, physiological processes, and neural dynamics, \cblue{the underlying dynamics are often highly nonlinear, high-dimensional, and insufficiently understood to permit a physics-informed or explicitly parameterized model.} In such settings, traditional filtering methods cannot be applied in their standard form and must instead resort to data-driven approaches that learn the system dynamics directly from observations. This shift significantly complicates the task of achieving accurate and robust state estimation \cite{revach2022kalmannet}.

A variety of data-driven approaches for modeling dynamical systems have been proposed. Deterministic models, such as deep state-space models (DSSMs), employ neural networks to learn system dynamics and have gained popularity for their ability to capture high-dimensional, nonlinear behaviors \cite{girin2021dynamical}. However, these models often require large training datasets to avoid overfitting, and their generalization performance may degrade in uncertain or data-scarce environments. Moreover, their predictions can be difficult to interpret, especially when compared to more structured probabilistic models \cite{krishnan2017structured}. As a result, DSSMs may face challenges in practical deployment despite their modeling capacity \cite{ghosh2024danse}.

Stochastic models based on flexible random process priors provide a principled alternative within a Bayesian learning framework \cite{williams2006gaussian}. A prominent example is the class of Gaussian process state-space models (GPSSMs), which leverage Gaussian processes (GPs) to encapsulate uncertainty and provide implicit regularization through prior structure. These properties make GPSSMs particularly attractive for applications where uncertainty quantification and data efficiency are critical \cite{frigola2015bayesian,suwandi2023sparsityaware}. GPSSMs have been applied in diverse domains, including human pose and motion learning \cite{wang2007gaussian}, robotics and control \cite{deisenroth2013gaussian}, reinforcement learning \cite{arulkumaran2017deep,yan2020gaussian}, target tracking and navigation \cite{xie2020learning}, and magnetic-field sensing \cite{berntorp2023constrained}.

\cblue{In parallel to their practical applications, significant research efforts have been dedicated to advancing the learning and inference algorithms that underpin GPSSMs.} Early works either assumed a pre-learned GPSSM and concentrated solely on latent state inference—an impractical assumption for many real-world applications \cite{ko2009gp, deisenroth2009analytic, deisenroth2011robust}—or employed maximum \textit{a posteriori} (MAP) estimation to jointly infer the model and the latent states \cite{ wang2007gaussian,ko2011learning}, which is prone to overfitting. To address these issues, a important development came with the fully Bayesian treatment of GPSSMs using particle Markov chain Monte Carlo (PMCMC) methods \cite{frigola2013bayesian}, though the inherent cubic computational complexity of GP remains prohibitive for long time series. Subsequent efforts adopted reduced-rank GP approximations to alleviate the computational demands \cite{svensson2016computationally, svensson2017flexible, berntorp2021online,liu2023sequential}, yet PMCMC scalability remains a challenge. This has led to a growing interest in variational inference methods \cite{frigola2015bayesian,frigola2014variational, mchutchon2015nonlinear, eleftheriadis2017identification, doerr2018probabilistic, ialongo2019overcoming, curi2020structured, liu2020gpssm,lindinger2022laplace, fan2023free,lin2023towards,lin2023ensemble,lin2023towards_efficient,lin2022output} that can leverage sparse GP approximations with inducing points to enhance scalability \cite{titsias2009variational,hensman2013gaussian}.

However, existing GPSSMs, which employ multiple independent GPs to model system dynamics, remain limited to low-dimensional state spaces and face fundamental challenges when scaling to high-dimensional systems. Specifically, in variational GPSSMs, each output dimension of the state transition function is modeled by a separate GP, leading to a total computational cost that scales linearly with the state dimensionality. Since each GP has a cubic complexity in the number of data points due to covariance matrix inversion, the overall computational burden becomes prohibitive. Moreover, the number of model parameters grows quadratically with the state dimension, further exacerbating scalability issues. Additionally, all GPSSMs rely on stationary GP priors, which assume that the statistical properties of the transition function are invariant under input translation. This assumption often fails for systems exhibiting time-varying or non-stationary dynamics, limiting the expressiveness and practical applicability of conventional GPSSMs.

\cblue{To this end, we propose a novel GPSSM framework for non-stationary, high-dimensional dynamical systems, whose primary contribution is an efficient, flexible implicit process prior for the state transition function. By unifying Gaussian processes, input-dependent normalizing flows \cite{kobyzev2020normalizing}, and (Bayesian) neural networks, this prior is inherently non-stationary while maintaining computational and parametric complexity that scales gracefully with state dimensionality. The key contributions are summarized as follows:} 
\begin{itemize}
    \item Instead of employing multiple independent GPs for the transition function, we propose applying an efficient transformed Gaussian process (ETGP) \cite{maronas2022efficient}, an implicit random process that combines a shared GP with normalizing flows. Crucially, the flow parameters are made input-dependent through neural network modeling. This novel modeling approach offers two key advantages: (1) it captures complex, non-stationary transition functions that standard GP priors cannot represent, and (2) by utilizing only a single shared GP, it mitigates the computational and parametric complexities of traditional GPSSMs in high-dimensional spaces, enabling effective scaling to higher dimensions.

    \item For efficient inference, we develop a scalable variational algorithm to jointly approximate the posterior distributions of both the ETGP and the latent states.  Unlike standard GPs, the implicit nature of the ETGP prevents direct parametrization of the function-space posterior. To address this challenge, we explicitly approximate the posterior over the underlying GP and the associated neural network parameters, which together define the ETGP posterior. Additionally, we derive a generic and assumption-free variational lower bound that can incorporate an EnKF into the variational inference framework, eliminating the need to parameterize the variational distributions over latent states. This hybrid formulation thus not only offers accurate state estimation but also substantially reduces the computational cost of the parameter optimization.

    \item Finally, extensive empirical evaluations across system dynamics learning, high-dimensional state estimation, and time-series prediction demonstrate the superiority of our proposed method. It achieves significantly lower computational and parametric complexity than conventional GPSSMs while outperforming  state-of-the-art approaches, including deep neural network-based SSMs, in terms of learning and inference accuracy.
\end{itemize}

Part of this work was presented at IEEE ICASSP 2024 \cite{lin2023towards_efficient}. The current paper introduces input-dependent normalizing flows to further enhance model flexibility and proposes an EnKF-aided variational inference algorithm for efficient state estimation. More extensive empirical validations are also included to demonstrate the superiority of the proposed method.

The remainder of this paper is structured as follows. Section~\ref{sec:preliminaries_and_bg} introduces the preliminaries of state-space models and outlines key challenges in existing GPSSMs. Section~\ref{sec:ETGPSSM} presents our proposed efficient GPSSM framework for high-dimensional dynamical systems, including the associated variational inference algorithm. Experimental results are detailed in Section~\ref{sec:experiments}, followed by concluding remarks in Section~\ref{sec:conclusions}. Supporting results and proofs are deferred to the Appendix and supplementary materials.  
\section{Preliminaries}
\label{sec:preliminaries_and_bg}

Section~\ref{subsec:SSMs} introduces the SSM considered in this work. A brief review of GPSSM is presented in Section~\ref{subsec:GPSSMs}. The challenges associated with existing GPSSMs are discussed in detail in Section~\ref{subsec:challenges_GPSSMs}.

\subsection{State-Space Models (SSMs)} \label{subsec:SSMs}

An SSM defines a probabilistic dynamical system by specifying how latent states $\x_t \in \mathbb{R}^{d_x}$ evolve over time and generate observations $\y_t \in \mathbb{R}^{d_y}$. It is commonly expressed as:
\begin{subequations}
    \label{eq:SSM}
    \begin{align}
        \text{\textit{(Transition)}:}  \quad  &\x_{t+1} = \bm{f}(\x_t) + \mathbf{v}_t,  & \mathbf{v}_t \sim \cN(\bm{0}, \mathbf{Q}), \label{eq:ssm_transition}\\
        \text{\textit{(Emission)}:}   \quad &\y_{t} = \bm{C} \x_t + \mathbf{e}_t,	& \mathbf{e}_t \sim \cN(\bm{0}, \mathbf{R}).
    \end{align}
\end{subequations}
The latent states evolve according to a first-order Markov process, meaning that at any time step $t \in \mathbb{N}$, the next state $\x_{t+1}$ depends only on the current state $\x_t$ through the transition function $\bm{f}(\cdot): \mathbb{R}^{d_x} \mapsto \mathbb{R}^{d_x}$. In this work, the emission function is assumed to be linear, parameterized by a matrix $\bm{C} \in \mathbb{R}^{d_y \times d_x}$. For systems with nonlinear emissions, the latent state can be augmented to a higher-dimensional representation, effectively rendering the emission function linear \cite{doerr2018probabilistic}. This reparameterization helps mitigate non-identifiability issues frequently encountered in data-driven SSMs \cite{frigola2015bayesian}. Both latent states and observations are perturbed by zero-mean Gaussian noise, with covariance matrices $\mathbf{Q}$ and $\mathbf{R}$, respectively.

\begin{figure}[t!]
    \centering
    \footnotesize
    \begin{tikzpicture}[align = center, latent/.style={circle, draw, text width = 0.45cm}, observed/.style={circle, draw, fill=gray!20, text width = 0.45cm}, transparent/.style={circle, text width = 0.45cm}, node distance=1.2cm]
        \node[latent](x0) {${\x}_0$};
        \node[latent, right of=x0, node distance=1.4cm](x1) {${\x}_{1}$};
        \node[transparent, right of=x1](x2) {$\cdots$};
        \node[latent, right of=x2](xt-1) {$\!\!{\x}_{t-1}\!\!$};
        \node[latent, right of=xt-1, node distance=1.4cm](xt) {${\x}_{t}$};
        \node[transparent, right of=xt](xinf) {$\cdots$};
        \node[transparent, above of=x0](f0) {$\cdots$};
        \node[latent, above of=x1](f1) {${\f}_{1}$};
        \node[transparent, right of=f1](f2) {$\cdots$};
        \node[latent, above of=xt-1](ft-1) {$\!\!{\f}_{t-1}\!\!$};
        \node[latent, above of=xt](ft) {${\f}_{t}$};
        \node[transparent, right of=ft](finf) {$\cdots$};
        \node[observed, below of=x1](y1) {${\y}_{1}$};
        \node[transparent, below of=x2](y2) {$\cdots$};
        \node[observed, below of=xt-1](yt-1) {$\!\!{\y}_{t-1}\!\!$};
        \node[observed, below of=xt](yt) {${\y}_{t}$};
        \node[transparent, right of=yt](yinf) {$\cdots$};
        \draw[-latex] (x0) -- (f1);
        \draw[-latex] (f1) -- (x1);
        \draw[-latex] (x1) -- (f2);
        \draw[-latex] (ft-1) -- (xt-1);
        \draw[-latex] (xt-1) -- (ft);
        \draw[-latex] (x2) -- (ft-1);
        \draw[-latex] (ft) -- (xt);
        \draw[-latex] (xt) -- (finf);
        \draw[-latex] (x1) -- (y1);
        \draw[-latex] (xt-1) -- (yt-1);
        \draw[-latex] (xt) -- (yt);
        \draw[ultra thick]
        (f0) -- (f1)
        (f1) -- (f2)
        (f2) -- (ft-1)
        (ft-1) -- (ft)
        (ft) -- (finf)
        ;
    \end{tikzpicture}
    \caption{Graphical representation of GPSSM. The white circles represent the latent variables, while the gray circles represent the observable variables. The thick horizontal bar represents a set of fully connected, mutually correlated nodes, specifically, the GP.}
    \label{fig:graphical_model_GPSSM}
\end{figure}

\subsection{Gaussian Process State-Space Models (GPSSMs)} 
\label{subsec:GPSSMs}

As discussed in Section~\ref{sec:intro},  \cblue{system dynamics in real-world complex scenarios are often unknown, but only a set of noisy observations, $\y_{1:T} \!=\! \{\y_t\}_{t=1}^T$, where $T \in \mathbb{Z}_{+}$, is available.} Consequently, there has been growing interest in adopting data-driven approaches, such as GPs \cite{williams2006gaussian}, to model these dynamics and quantify the underlying uncertainties.

\vspace{.01in}
\subsubsection{\underline{\textbf{Gaussian process (GP)}}}  \label{subsubsec:GPs}
The GP is a generalization of the Gaussian distribution over infinite index sets, enabling the specification of distributions over functions $\tilde{f}(\cdot) : \mathbb{R}^{d_x} \mapsto \mathbb{R}$.  
\begin{definition}
A stochastic process \(\{\tilde{f}(\mathbf{x}) : \mathbf{x} \in \mathcal{X}\}\) indexed by a domain \(\mathcal{X} \subseteq \mathbb{R}^{d_x}\) is said to be a GP \cite{williams2006gaussian} if, for every finite set of inputs $\mathbf{X} = \{\x_i\}_{i=1}^N \subset \mathbb{R}^{d_x}$, the corresponding random vector $\tilde{\f} = \{\tilde{f}(\x_i)\}_{i=1}^N$ is jointly Gaussian. A GP is fully characterized by its mean function, often set to zero, and its covariance (kernel) function $k(\x, \x^\prime)$, which includes a set of hyperparameters $\bm{\theta}_{gp}$ \cite{williams2006gaussian}. We denote this as  
\[
\tilde{f}(\x) \sim \mathcal{GP}\big(0,\, k(\x, \x^\prime \,; \btheta_{gp})\big).
\]
\end{definition}
\cblue{A widely used family of kernels is the class of stationary kernels \cite{williams2006gaussian}, which is formally defined as follows.
\begin{definition}
    A kernel $k$ is stationary if there exists a function $\kappa: \mathbb{R}^{d_x} \times \mathbb{R}^{d_x}  \to \mathbb{R}$, such that $k(\mathbf{x}, \mathbf{x}') = \kappa(\mathbf{x} - \mathbf{x}')$ for all $\x, \x^\prime \subset \cal X$.
\end{definition}
This assumption encodes translation-invariant properties and is shared by common kernels such as the squared exponential (SE) and Matérn families \cite{williams2006gaussian}. Together with a constant mean function, a stationary kernel defines a stationary GP, whose finite-dimensional distributions are invariant to translations of the input space. Specifically, for any inputs $\mathbf{X}$, the distribution
\begin{equation}
   \label{eq:GP_finite}
   p(\tilde{\f} \mid \mathbf{X}) = \mathcal{N}(\tilde{\f} \mid \bm{0}, \, \mathbf{K}), 
\end{equation}
where $\mathbf{K} \in \mathbb{R}^{N \times N}$ and $[\mathbf{K}]_{i,j} = k(\x_i, \x_j\,; \btheta_{gp})$, remains unchanged under translation.
}

Given function values $\tilde{\f}$, the GP prediction at a new input $\x_* \!\subset\! \mathcal{X}$, denoted $p(\tilde{f}(\x_*) \vert \tilde{\f}, \mathbf{X})$, is also Gaussian, with posterior mean $\xi(\x_*)$ and posterior variance $\Xi(\x_*)$ given by:
\begin{subequations}
\label{eq:GP_posterior}
    \begin{align}
        & \xi(\mathbf{x}_*) = \mathbf{K}_{\mathbf{x}_*, \mathbf{X}} \mathbf{K}^{-1} \tilde{\f}, 
        \label{eq:post_mean}\\
        & \Xi(\mathbf{x}_*) = k(\mathbf{x}_*, \mathbf{x}_*)  - \mathbf{K}_{\mathbf{x}_*, \mathbf{X}} \mathbf{K}^{-1} \mathbf{K}_{\mathbf{x}_*, \mathbf{X}}^\top,
        \label{eq:post_cov}
    \end{align}
\end{subequations}
where $\mathbf{K}_{\x_*, \mathbf{X}}$ is the cross-covariance vector between the test input $\x_*$ and the training inputs $\mathbf{X}$.


\vspace{.01in}
\subsubsection{\underline{\textbf{Gaussian process state-space model}}}
\label{subsubsec:gpssms}
Placing a GP prior over the transition function $\bm{f}(\cdot)$ in the classic SSM (see Eq.~\eqref{eq:SSM}) yields the GPSSM \cite{frigola2015bayesian}:
\begin{subequations}
    \label{eq:gpssm}
    \begin{align}
        &  \bm{f}(\cdot)  \sim \mathcal{GP} \left(\bm{0}, \bm{k}(\cdot, \cdot)\right), \label{eq:gp_in_gpssm}\\
        & {\f}_{t} \triangleq \bm{f}(\mathbf{x}_{t-1}), \\
        & \mathbf{x}_{0} \sim p(\mathbf{x}_{0}),  \\
        &  \mathbf{x}_{t} \vert {\f}_{t}  \sim \mathcal{N}\left( \x_t \vert {\f}_{t}, \mathbf{Q}\right), \\
        &  \mathbf{y}_{t} \vert  \mathbf{x}_{t} \sim \cN \left(\y_t \vert \bm{C} \mathbf{x}_{t}, \mathbf{R}\right),
    \end{align}
\end{subequations}
where the initial state $p(\x_0)$ is assumed known and Gaussian for simplicity. A graphical representation is shown in Fig.~\ref{fig:graphical_model_GPSSM}.

When the state dimension $d_x > 1$, the transition function $\bm{f}(\cdot): \mathbb{R}^{d_x} \mapsto \mathbb{R}^{d_x}$ is typically modeled using $d_x$ independent GPs. Each component function $f_d(\cdot): \mathbb{R}^{d_x} \mapsto \mathbb{R}$ maps to one output dimension:
\begin{equation}
\f_t = \bm{f}(\x_{t-1}) = \{{f}_d(\x_{t-1})\}_{d=1}^{d_x} = \{\mathrm{f}_{t,d}\}_{d=1}^{d_x} , 
\label{eq:multivariate_GP}
\end{equation}
where each GP has its own kernel function and hyperparameters \cite{doerr2018probabilistic}. \cblue{Given a single observed sequence $\{\mathbf{y}_t\}_{t=1}^T$, the central task in GPSSMs involves jointly learning the transition dynamics and noise parameters, i.e., learning $[\btheta_{gp}, \mathbf{Q}, \mathbf{R}]$, while simultaneously inferring both the corresponding latent states and the underlying functions.}

\subsection{Challenges in Existing Variational GPSSMs}
\label{subsec:challenges_GPSSMs}
\cblue{In contrast to the standard GP described in Eq.~\eqref{eq:GP_posterior}, where the GP posterior is computed directly from observed input-output pairs, inferring the GP in a GPSSM is more complex. The key challenge is that the inputs to the transition function are the latent states $\{\mathbf{x}_{t}\}_{t=0}^{T-1}$, which are themselves unobserved. Consequently, the GP posterior over the transition function must be inferred jointly with the latent state trajectory. This is typically done approximately via variational inference \cite{frigola2014variational}. The variational GPSSMs} employ sparse GP approximations \cite{titsias2009variational,hensman2013gaussian}, a widely used technique in GP-based models \cite{maronas2021transforming,frigola2015bayesian} due to its analytical tractability and reduced computational cost. The core idea is to augment the GP with a smaller set of inducing points, which act as a sparse approximation of the full GP. 

More specifically, we can introduce $M \!\ll\! N$ inducing inputs $\mathbf{Z} = \{\z_i\}_{i=1}^M \subset \mathbb{R}^{d_x}$ and their corresponding inducing outputs $\u =\{\mathrm{u}_i\}_{i=1}^M = \{\tilde{f}(\z_i)\}_{i=1}^M$, which inherit the same GP prior as the original process, i.e., $\u \sim \mathcal{N}(\bm{0}, \mathbf{K}_{\mathbf{Z}, \mathbf{Z}})$. The augmented GP prior can be written as:\footnote{With slight notation abuse, we denote the distribution of GP prior as $p(\tilde{f})$.}
\begin{equation}
    p(\tilde{f}, \u) = p(\tilde{f} \vert \mathbf{u}) p(\mathbf{u}).
\end{equation}
To approximate the posterior, a variational distribution is defined over the $M$ inducing outputs, $q(\u) \!=\! \mathcal{N}(\u \vert \mathbf{m}, \mathbf{S})$, while the prior conditional is used for the rest of the GP:
\begin{equation}
    q(\tilde{f}, \u) =  p(\tilde{f} \vert \mathbf{u}) q(\mathbf{u}).
\end{equation}
This leads to an explicit expression for the variational approximation to the GP posterior:
\begin{align} 
\label{Eq:variational_GP_poserior}
    q(\tilde{f}) = \int q(\tilde{f}, \u) \mathrm{d} \u = \cN(\tilde{f}(\x_*) \mid \xi_p(\x_*), \, \Xi_p(\x_*)),
\end{align}
where
\begin{subequations}
\label{eq:variational_GP_poserior_details}
\begin{align}
    &  \xi_p(\x_*) \!=\! \mathbf{K}_{\x_{*}, \mathbf{Z}}\mathbf{K}_{\mathbf{Z},\mathbf{Z}}^{-\!1} \mathbf{m}, \\
    & \Xi_p(\x_*) \!=\! \mathbf{K}_{\x_{*}, \x_{*}}  \!-\! \mathbf{K}_{\x_{*}, \mathbf{Z}}\mathbf{K}_{\mathbf{Z}, \mathbf{Z}}^{-\!1}[\mathbf{K}_{\mathbf{Z}, \mathbf{Z}} \!-\! \mathbf{S}]\mathbf{K}_{\mathbf{Z}, \mathbf{Z}}^{-\!1}\mathbf{K}_{\x_{*}, \mathbf{Z}}^\top,
\end{align}
\end{subequations}
reducing the computational complexity from $\mathcal{O}(N^3)$ to $\mathcal{O}(M^3)$. 
Nevertheless, existing variational GPSSMs face several computational and modeling challenges, outlined below:
\begin{itemize}
    \item Modeling multi-dimensional state transition function using separate GPs for each output dimension results in a computational complexity of $\mathcal{O}(d_x M^3)$, scaling linearly with the latent state dimension $d_x$.
    \item The variational approximation introduces a number of variational parameters that scale quadratically with the state dimensionality $d_x$. Specifically, variational GPSSMs define inducing inputs $\{\mathbf{Z}_d\}_{d=1}^{d_x} \in \mathbb{R}^{d_x \times (M \times d_x)}$, variational means $\{\mathbf{m}_d\}_{d=1}^{d_x} \in \mathbb{R}^{d_x \times M}$, and covariance matrices $\{\mathbf{S}_d\}_{d=1}^{d_x} \in \mathbb{R}^{d_x \times (M \times M)}$ for the $d_x$-dimensional inducing outputs $\{\mathbf{u}_{d}\}_{d=1}^{d_x}$. As the state dimension increases, the quadratic scaling of inducing inputs and the associated variational parameters leads to significant computational overhead \cite{lin2023towards_efficient}.
    \item \cblue{Most GPSSMs impose a stationary GP prior on the transition function, as shown in Eq.~\eqref{eq:gp_in_gpssm}. However, this stationarity assumption may be inadequate for capturing time-varying or input-dependent dynamical behavior. 
    Additionally, modeling each output dimension with an independent GP ignores potential correlations across state dimensions, which can limit the model's capacity.
    }
\end{itemize} 
\cblue{It is important to clarify that the ``stationarity'' in this work refers specifically to the properties of the transition function prior, not the observed time series $\mathbf{y}_{1:T}$. A non-stationary transition function is one whose characteristics (e.g., amplitude, smoothness, or periodicity) vary with the input state. Notably, a GPSSM with a stationary kernel in its transition function can still generate a non-stationary time series.}
In the next section, we present a methodology to resolve these issues.

\section{Efficient Transformed Gaussian Process State-Space Models} \label{sec:ETGPSSM}

We first present our proposed efficient and flexible modeling framework for high-dimensional state spaces in Section~\ref{subsec:ETGPSSM_modeling}. Next, Section~\ref{subsec:learning_inference} introduces the corresponding EnKF-aided variational inference algorithm. Finally, Section~\ref{subsec:EnK_ELBO_evaluation} details the evaluation of the EnKF-aided variational lower bound.

\subsection{Efficient Modeling for High-Dimensional State Space} \label{subsec:ETGPSSM_modeling}

\begin{figure*}
    \centering
    \subfloat[\cblue{Limited modeling capacity of the warped GPs (Eq.~\eqref{eq:ETGP}).}]{
    \includegraphics[width=0.48\linewidth]{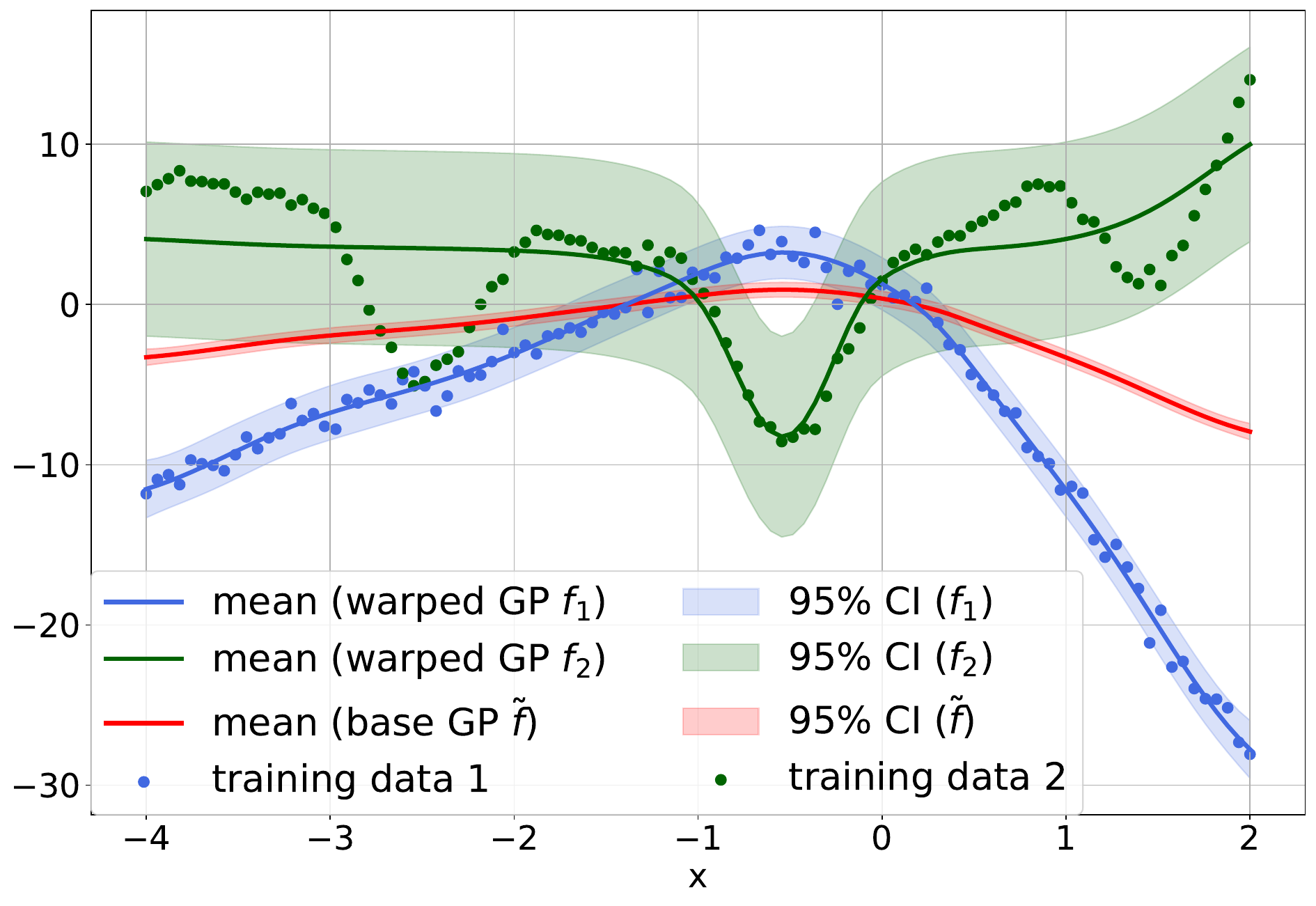}
    \label{subfig:etgp_demo}
    }  
    \subfloat[\cblue{Independent GPs versus ETGPs (Eq.~\eqref{eq:time_varying_linear_flow}).}]{
    \includegraphics[width=0.485\linewidth]{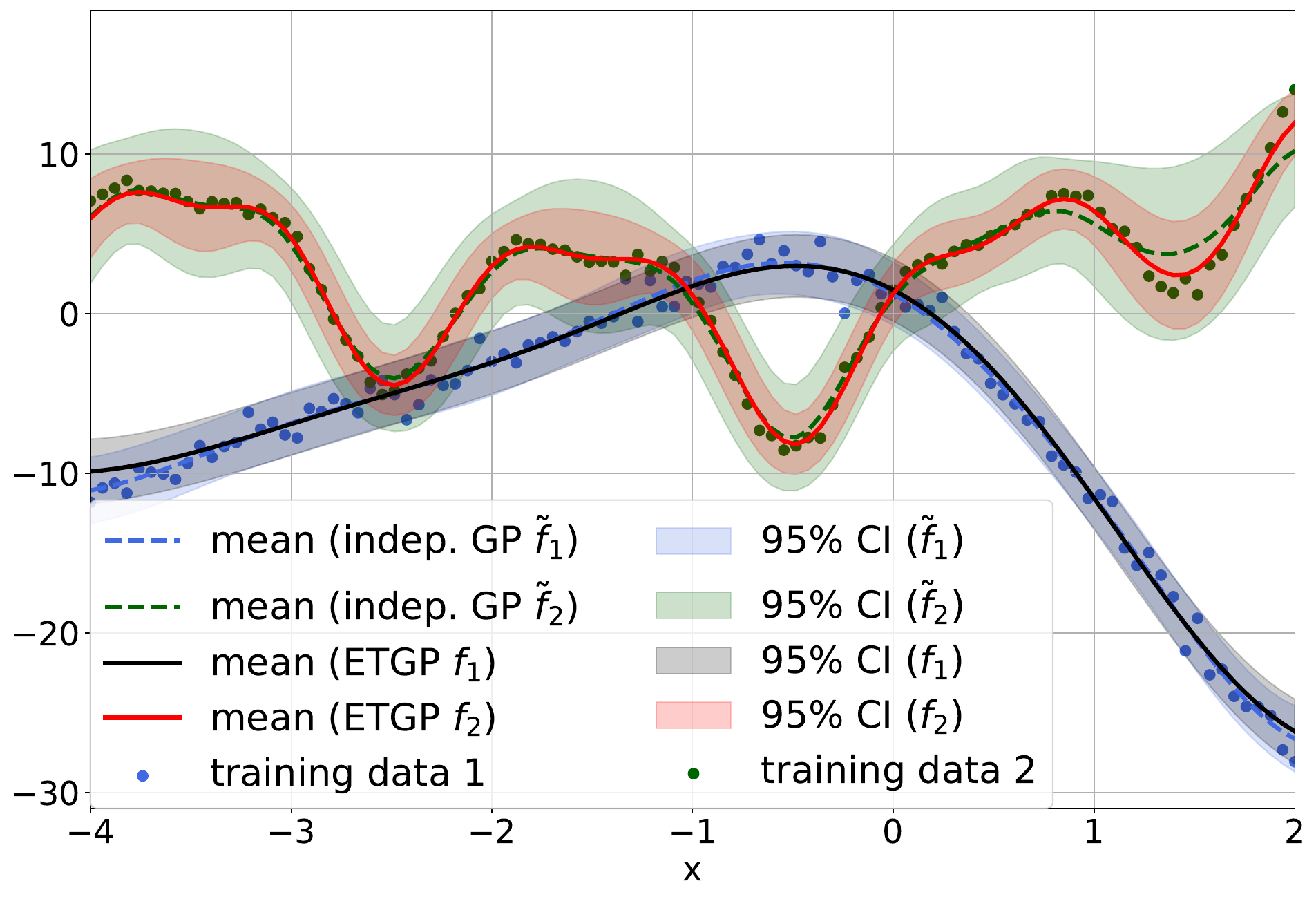}
    \label{subfig:etgp_demo_flow}
    }
    \caption{Two-dimensional regression example comparing GP models. The non-stationary ETGP accurately captures local variations, outperforming stationary warped GPs and independent GPs.
    \vspace{-.1in}
    }
    \label{fig:ETGP_demo_inability}
\end{figure*}

In Section~\ref{subsec:challenges_GPSSMs}, we discussed the computational and modeling challenges that arise from using separate GPs for each output dimension of the transition function in high-dimensional state spaces. To address these limitations, \cite{lin2023towards_efficient} proposed warping a single shared GP using normalizing flows \cite{kobyzev2020normalizing} to model each output dimension of the transition function:
\begin{align}
\label{eq:ETGP}
\Tilde{f}(\cdot) \!\sim\! \mathcal{GP}(0, k(\cdot, \cdot)),~ 
 f_d(\cdot) \!=\! \G_{\btheta_d}(\Tilde{f}(\cdot)),~  d \!=\! 1,\ldots, d_x,
\end{align}
where each dimension-specific normalizing flow, ${\G}_{\bm{\theta}_d}(\cdot)$, is an element-wise, bijective, and differentiable transformation parameterized by $\bm{\theta}_d \in \Theta$. Two concrete examples are presented below.

\begin{example} \label{example:linear_flow}
     If $\G_{\btheta_{d}}(\cdot)$ is a simple linear flow, then we have
    \begin{equation}
     f_{d}(\cdot) = \alpha_d \cdot \Tilde{f}(\cdot) + \beta_d, \nonumber
    \end{equation}
    with $\btheta_d = [\alpha_d, \beta_d]^\top \in \mathbb{R}^2$. 
\end{example}
\begin{example}
\label{example:SAL_flows}
If $\G_{\btheta_{d}}(\cdot)$ is a Sinh-Arcsinh-Linear (SAL) flow \cite{maronas2021transforming,rios2019compositionally}, we have:
\begin{equation}
    \label{eq:SAL_flows}
    f_d(\cdot)= \alpha_{d} \sinh \left[\varphi_{d} \operatorname{arcsinh}\left( \tilde{f}(\cdot) \right) - \gamma_{d} \right]+\beta_{d}, \nonumber
\end{equation}
with $\btheta_{d} = [\alpha_{d}, \beta_{d}, \gamma_{d}, \varphi_{d}]^\top\in \mathbb{R}^4$.
\end{example}
More complex flows can be achieved by stacking additional layers of SAL flow and/or incorporating more advanced deep neural network-based flows, such as RealNVP \cite{dinh2017density}.

The main advantage of this approach lies in eliminating the need for multiple independent GPs by using computationally efficient flow transformations. This design bypasses the $\mathcal{O}(d_x M^3)$ computational complexity and quadratic parameter growth discussed in Section~\ref{subsec:challenges_GPSSMs}. Moreover, since all dimensions share the same underlying GP, the resulting warped functions are inherently dependent, potentially mitigating the limitations of modeling independent outputs in standard GPSSMs.

However, the warped GP in Eq.~\eqref{eq:ETGP} remains stationary; Also, being derived from a single shared GP, it may lack the expressiveness needed to capture complex latent dynamics—resulting in suboptimal performance in practice \cite{lin2023towards_efficient}.
\cblue{A demonstrative example is provided as follows.
\begin{example}
\label{example:regression_noisy_kink}
    Consider a two-dimensional regression problem where the outputs are generated as follows:
    \begin{align*}
        y_1 &= 3.5 \cdot h(x), \\
        y_2 &= -1.5 \cdot h(x) + 5 \sin(\pi x) + 2 \cos(2\pi x),
    \end{align*}
    where \( h(x) = 0.8 + (x + 0.2) \left(1 - \frac{5}{1 + \exp(-2x)}\right) + \varepsilon \) is a noisy kink function, and noise \( \varepsilon \sim \mathcal{N}(0, 0.25^2) \). The function  \( y_1 \) is linearly dependent on the noisy kink function, and \( y_2 \) introduces additional nonlinear, periodic components and a negative correlation with \( y_1 \). A dataset of 100 input-output pairs was generated with inputs in \( [-4, 2] \), as shown in Fig.~\ref{fig:ETGP_demo_inability}. The warped GPs, which are constructed by a base GP with an SE kernel and two SAL flows, were trained by maximizing the variational lower bound, see e.g. \cite[Section~4.1]{maronas2021transforming} for more details.

    The results in Fig.~\ref{subfig:etgp_demo} show that the learned element-wise SAL flow \( \mathbb{G}_{\bm{\theta}_1} \) successfully warps the base GP \( \tilde{f} \) into \( f_1 \), which fits \( y_1 \) well. In contrast, applying \( \mathbb{G}_{\bm{\theta}_2} \) to the same base process yields \( f_2 \), which fails to capture the complex structure of \( y_2 \) due to the limited modeling flexibility of the input-independent SAL flows (see Example~\ref{example:SAL_flows}), resulting in high predictive uncertainty.
\end{example}
}

To overcome the limitations in expressiveness, we draw inspiration from recent advances in adaptive parameterization \cite{maronas2021transforming, gu2023mamba} and introduce non-stationarity and input dependence into the warped GPs. Specifically, we generate non-stationary processes $\{f_d(\cdot)\}_{d=1}^{d_x}$ by conditioning the flow parameters on the input through a shallow neural network $\textsc{nn}: \mathbb{R}^{d_x} \mapsto \Theta$ parameterized by weights $\mathbf{w}$. In the SSM context, this is formulated as\footnote{Note that $\x_t$ can always be extended to accommodate control systems incorporating a deterministic control input $\bm{c}_t \in \mathbb{R}^{d_c}$ by augmenting the latent state to $[\x_t^\top, \bm{c}_t^\top]^\top \in \mathbb{R}^{d_x + d_c}$. For brevity, however, we omit explicit reference to $\bm{c}_t$ in our notation throughout this paper.}:
\begin{equation} \label{eq:nn_flow_para}
    \btheta_t = \textsc{nn}_{\mathbf{w}}(\x_{t-1}), 
\end{equation}
where $\btheta_t = \{\theta_{t,d}\}_{d=1}^{d_x}$. This formulation allows the flow parameters to vary with input at each time step $t$, enabling the model to capture local variations and complex dynamics. Additionally, note that the parameter amortization learning shown in Eq.~\eqref{eq:nn_flow_para} prevents the number of flow parameters from growing with time step $t$ while also improving parameter efficiency through weight sharing across dimensions.

\begin{figure}
    \centering
    \includegraphics[width=\linewidth]{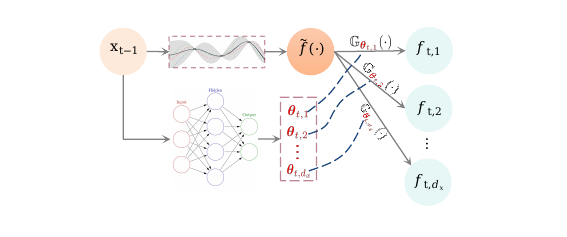}
    \caption{Non-stationary ETGP \cblue{transition} function in SSMs}
    \label{fig:ETGP}
\end{figure}

We may further impose a prior $p_{\bm{\psi}}(\mathbf{w})$ over $\mathbf{w}$ with hyperparameter $\bm{\psi}$ (e.g., $\mathbf{w} \!\sim\! \cN(\mathbf{w} \vert \bm{0}, \operatorname{diag}(\bm{\psi}))$) to regularize the neural network's parameters and mitigate overfitting issues. In summary, this efficient transformed GP modeling can be mathematically expressed as follows:
\begin{subequations}
\label{eq:ETGP_bayesian}
\begin{align} 
    & \Tilde{f}(\cdot) \sim \mathcal{GP}(0, k(\cdot, \cdot)), \ \ \mathbf{w} \sim p_{\bm{\psi}}(\mathbf{w}), \\
    & \btheta_t = \textsc{nn}_{\mathbf{w}}(\x_{t-1}), \ \ f_d(\cdot) = \G_{\theta_{t,d}}(\Tilde{f}(\cdot)),  \ d\!=\!1,..., {d_x},
\end{align} 
\end{subequations}
 where $\{f_d(\cdot)\}_{d=1}^{d_x}$ implicitly characterize $d_x$ random processes \cite{ma2019variational}. 
 To better understand the modeling, Fig.~\ref{fig:ETGP} illustrates the corresponding graphical representation with $p_{\bm{\psi}}(\mathbf{w}) = \delta(\mathbf{w} - \bar{\mathbf{w}})$, where $\delta(\cdot)$ is the Dirac delta measure and $\bar{\mathbf{w}}$ is a deterministic vector. The following example extends the Example~\ref{example:linear_flow}, providing a more concrete illustration.
 \begin{example} \label{example:time_varying_linear_flow}
    In the case of the linear flow, we have
    \begin{subequations}
    \label{eq:time_varying_linear_flow}
    \begin{align}
        & \bm{f}(\cdot) = \bm{\alpha}_{t} \cdot \tilde{f}(\cdot) + \bm{\beta}_{t},\\
        & \btheta_{t} = \textsc{nn}_{\mathbf{w}}(\x_{t-1}), \ \ \ \mathbf{w}\sim p_{\bm{\psi}}(\mathbf{w}).
    \end{align}
    \end{subequations}
    where $\btheta_{t} \!=\! [\bm{\alpha}_{t}^\top, \bm{\beta}_{t}^\top]^\top$, and $\bm{\alpha}_{t} \!=\! \{\alpha_{t,d}\}_{d=1}^{d_x}$, $\bm{\beta}_{t} \!=\! \{\beta_{t,d}\}_{d=1}^{d_x}$. 
    
    \cblue{An illustrative two-dimensional regression example using the resulting transformed GPs, ${f}_1$ and ${f}_2$, is shown in Fig.~\ref{subfig:etgp_demo_flow}, where the neural network weights follow a Dirac delta measure. The training data consist of 100 input–output pairs consistent with Example~\ref{example:regression_noisy_kink}. For comparison, results from two independent GPs, $\tilde{f}_1$ and $\tilde{f}_2$, with SE kernels are also presented. Both models were trained by maximizing their respective variational lower bounds \cite{hensman2013gaussian, maronas2022efficient}.

    While both models capture the general structure of the target functions, the independent GP $\tilde{f}_2$ exhibits a poorer fit and overestimates uncertainty for the more complex $y_2$, primarily due to its stationary SE kernel's inability to adapt to local variations. In contrast, the ETGP achieves a notably better fit for these complex patterns through its input-dependent linear flow, which induces non-stationary behavior. Moreover, the ETGP shares a single latent GP across outputs, offering improved parameter efficiency and scalability compared to modeling each output with a separate GP, as discussed in Section~\ref{subsec:challenges_GPSSMs}.}
\end{example}

It is noteworthy that this efficient transformed Gaussian process (ETGP) modeling strategy was originally proposed to tackle large-scale multi-class classification problems \cite{maronas2022efficient}. Its efficiency arises from employing a single GP instead of one per class, ensuring that the computational cost of GP operations remains independent of the number of classes. This results in a substantial reduction in overall computation. In this work, we adapt this framework to model high-dimensional transition functions in SSMs. Notably, ETGP implicitly models $d_x$ dependent random processes \cite{ma2019variational}, introducing output correlations that are often neglected in existing GPSSM formulations. We elaborate further on this below.
\begin{corollary}\label{corollary:dependent_nonstationary_GP}
    If $p_{\bm{\psi}}(\mathbf{w}) = \delta(\mathbf{w} - \bar{\mathbf{w}})$, where $\bar{\mathbf{w}}$ is a deterministic vector, then the ETGP defined in Example~\ref{example:time_varying_linear_flow} becomes a $d_x$-dimensional dependent non-stationary Gaussian process with \cblue{an input-dependent} covariance function. Specifically, at arbitrary time steps $t$ and $t^\prime$, we have
    \begin{equation}
        \begin{bmatrix}
            \bm{f}(\x_{t-1})\\
            \bm{f}(\x_{t^\prime-1})
        \end{bmatrix} \sim \cN\left( 
        \begin{bmatrix}
            \bm{\beta}_t\\
            \bm{\beta}_{t^\prime}
        \end{bmatrix}, 
        \bm{\Lambda}_{t, t^\prime}
        \right),
    \end{equation}
    where the covariance matrix $\bm{\Lambda}_{t, t^\prime} \in \mathbb{R}^{2d_x \times 2 d_x}$ is given by
    \begin{equation}
     \bm{\Lambda}_{t, t^\prime} \!=\! 
        \begin{bmatrix}
             k(\x_{t-1}, \x_{t-1}) \cdot \bm{\alpha}_t \bm{\alpha}_t^\top & k(\x_{t-1}, \x_{t^\prime-1}) \cdot \bm{\alpha}_t \bm{\alpha}_{t^\prime}^\top\\
             k(\x_{t^\prime-1}, \x_{t-1}) \cdot \bm{\alpha}_{t^\prime} \bm{\alpha}_t^\top & k(\x_{{t^\prime}-1}, \x_{{t^\prime}-1}) \cdot \bm{\alpha}_{t^\prime} \bm{\alpha}_{t^\prime}^\top
        \end{bmatrix}.
    \end{equation}
\end{corollary}
\begin{proof}
This result follows from the Gaussianity of the underlying process and the linearity of covariance, as detailed in Appendix~\ref{app:proof_corollary}.
\end{proof}

With the computationally efficient and flexible function prior defined by ETGP $\bm{f}(\cdot) \!=\! \{f_d(\cdot)\}_{d=1}^{d_x}$, we define our proposed SSM, termed ETGPSSM, as follows: 
\begin{subequations}
\label{eq:ETGPSSM}
\begin{align}
    & \bm{f}(\cdot) \sim \mathcal{ETGP}(\tilde{f}(\cdot), p_{\bm{\psi}}(\mathbf{w})), \\
    & \mathbf{x}_0 \sim p(\x_0), \\ 
    & \f_t \triangleq \bm{f}(\x_{t-1}), \\
    & \mathbf{x}_{t} \vert \f_{t}  \sim \mathcal{N}(\x_t \vert \f_{t}, \mathbf{Q}), \\
    & \mathbf{y}_{t} \vert  \mathbf{x}_{t} \sim \cN ( \y_t \vert \bm{C} \mathbf{x}_{t}, \mathbf{R}), \label{eq:linear_emission}
\end{align}
\end{subequations}
where the ETGP prior is defined by the base GP $\tilde{f}(\cdot)$, neural network weight prior $p_{\bm{\psi}}(\mathbf{w})$, and the normalizing flows, see Eq.~\eqref{eq:ETGP_bayesian}. Unless stated otherwise, in this work we focus on the linear flow described in Example~\ref{example:time_varying_linear_flow} for simplicity and illustration. This variant requires learning only two parameters per dimension, offering a more efficient alternative to more expressive but costly flow parameterizations. For simplicity, we also assume that the initial state prior $p(\x_0)$ is Gaussian and known.

\begin{remark}
In contrast to existing GPSSMs with independent, stationary transitions, the ETGPSSM yields dependent and non-stationary transition functions $\{f_d(\cdot)\}_{d=1}^{d_x}$ due to their shared base GP and weight priors. In the case of linear input-dependent flows, the transition function remains a GP—like a GPSSM—but enriched with dependence and non-stationarity, see Corollary~\ref{corollary:dependent_nonstationary_GP}, offering greater modeling flexibility while retaining GP interpretability.
\end{remark}

The primary objective in ETGPSSM, as in GPSSM, is the joint estimation of model hyperparameters $\bm{\vartheta} \!=\! \{\btheta_{gp}, \bm{\psi}, \mathbf{Q}, \mathbf{R}\}$ and inference of both the ETGP and latent states. Given a sequence of observations $\y_{1:T}$, the task is to estimate $\bm{\vartheta}$ and compute the posterior distributions $p(\bm{f} \vert \y_{1:T})$ for the ETGP and $p(\x_{0:T} \vert \y_{1:T})$ for the latent states, where $\x_{0:T} = \{\x_t\}_{t=0}^T$.

However, the marginal likelihood $p(\mathbf{y}_{1:T} \vert \bm{\vartheta})$ is generally intractable, making both model evidence maximization and posterior inference computationally challenging. Furthermore, unlike GPs, which admit closed-form expressions for their prior and (approximate) posterior (see Eqs.~\eqref{eq:GP_finite}, \eqref{eq:GP_posterior} \& \eqref{Eq:variational_GP_poserior}), the ETGP defines an implicit stochastic process. As a result, both the prior and posterior are analytically intractable \cite{ma2019variational}, presenting significant challenges for inference. In the following subsections, we introduce a variational inference framework that enables joint learning of model hyperparameters and approximate Bayesian inference over the ETGP and latent states.


\subsection{Approximate Bayesian Inference} \label{subsec:learning_inference}

\cblue{Variational inference proceeds by maximizing a more tractable} evidence lower bound (ELBO) on the log-marginal likelihood, which is equivalent to minimizing the Kullback–Leibler (KL) divergence between the variational approximation and the true posterior \cite{theodoridis2020machine}.  Constructing the ELBO requires specifying both the model joint distribution and the variational distribution. 

Following the existing GPSSM framework, \cblue{we factorize the joint distribution of the ETGPSSM defined in Eq.~\eqref{eq:ETGPSSM} as:}
\begin{equation}
\begin{aligned}
\!\!\!p(\y_{1:T}, \x_{0:T}, \bm{f}) &\!=\! p(\y_{1:T}, \x_{1:T} \vert \bm{f}, \x_0) p(\bm{f}) p(\x_0) \\
&\!=\! \prod_{t=1}^{T} p(\y_t \vert \x_t) p(\x_{t} \vert \x_{t-1}, \bm{f}) p(\bm{f}) p(\x_0),
\end{aligned}
\end{equation}
where $p(\x_0) = \cN(\x_0 \vert \bm{0}, \bm{I})$, and \cblue{under slight abuse of notation\footnote{In general a random process does not have density function.}, $p(\bm{f})$ denotes the ETGP prior, which is generally intractable because it requires marginalizing over the base GP $\tilde{f}$ and the neural network weights $\mathbf{w}$.
}
A common variational family in the GPSSM literature takes the form:
\begin{equation}
q(\x_{0:T}, \bm{f}) = q(\x_{1:T} \vert \bm{f}, \x_0) q(\bm{f}) q(\x_0),
\end{equation}
which yields the ELBO, denoted by $\mathcal{L}_2$:
\begin{equation}
\label{eq:ELBO_general_2}
\mathcal{L}_2 = \mathbb{E}_{q(\x_{0:T}, \bm{f})} \left[\log \frac{p(\y_{1:T}, \x_{0:T}, \bm{f})} {q(\x_{0:T}, \bm{f}) }\right].
\end{equation}
While general, this formulation requires explicit parameterization of the variational distribution  $q(\x_{1:T} \vert \bm{f}, \x_0)$, which can introduce a large number of parameters and increase optimization complexity \cite{lin2023ensemble,ialongo2019overcoming}. 

To improve learning efficiency, in this work, we instead adopt an alternative approach based on marginalizing out $\x_{1:T}$ and infer latent states using an EnKF, avoiding the need to explicitly model this variational term. Specifically,  marginalizing over $\x_{1:T}$ leads to the following joint distribution:
\begin{align}
   \!\!\! p(\y_{1:T}, \x_0, \bm{f}) &\!=\! p(\x_0)\, p(\bm{f}) \int p(\y_{1:T}, \x_{1:T} \vert \bm{f}, \x_0) \, \mathrm{d}\x_{1:T} \nonumber \\
    &\!=\! p(\x_0)\, p(\bm{f})\, p(\y_{1:T} \vert \x_0, \bm{f}),
    \label{eq:marginalized_joint_distribution}
\end{align}
where $p(\y_{1:T} \vert \x_0, \bm{f})$ will be approximated sequentially using an EnKF, as detailed in Section~\ref{subsec:EnK_ELBO_evaluation}. Since only $\x_0$ and $\bm{f}$ remain latent, we adopt the factorized variational distribution:
\begin{equation}
\label{eq:variational_distribution_marginalized_general}
    q(\x_0, \bm{f}) = q(\x_0)\, q(\bm{f}).
\end{equation}
This leads to the marginalized ELBO:
\begin{equation}
\label{eq:ELBO_general}
    \mathcal{L} = \mathbb{E}_{q(\x_0, \bm{f})} \left[ \log \frac{p(\y_{1:T}, \x_0, \bm{f})}{q(\x_0)\, q(\bm{f})} \right].
\end{equation}
Compared to $\mathcal{L}_2$, the alternative bound $\mathcal{L}$ avoids direct parameterization of $q(\x_{1:T} \vert \bm{f}, \x_0)$, helping reduce optimization overhead and yielding a tighter variational lower bound (see Proposition~\ref{prop:bound_comparison}).
\begin{proposition}  
\label{prop:bound_comparison}
The bound $\mathcal{L}$ in Eq.~\eqref{eq:ELBO_general} satisfies $\mathcal{L} \ge \mathcal{L}_2$. 
Equality holds if and only if the variational approximation $q(\x_{1:T} \vert \bm{f}, \x_0)$ matches the true posterior $p(\x_{1:T} \vert \bm{f}, \x_0, \y_{1:T})$.  
\end{proposition} 
\begin{proof}
    Proof can be found in Appendix~\ref{appx:bounds_comparision}
\end{proof}

Theoretically, this tighter lower bound can improve inference accuracy and convergence \cite{theodoridis2020machine}. Next, we proceed to derive the ELBO $\mathcal{L}$. 

As noted earlier, the lack of closed-form expressions for the ETGP’s prior and posterior prevents direct computation of $\mathcal{L}$ in function space. To overcome this, we leverage the generative construction of the ETGP—illustrated in Fig.~\ref{fig:ETGP}—and infer the posterior distributions over $\tilde{f}$ and $\mathbf{w}$, which collectively characterize the ETGP posterior. Specifically, by explicitly representing $\tilde{f}$ and $\mathbf{w}$, the marginalized joint distribution becomes:
\begin{equation}
\label{eq:ETGPSSM_joint}
     p( \y_{1:T}, \x_{0}, \mathbf{w}, \tilde{f})  =  p(\y_{1:T} \vert \x_0, \mathbf{w}, \tilde{f}) p(\x_0) p_{\bm{\psi}}(\mathbf{w}) p(\tilde{f}),
\end{equation}
where $p(\tilde{f})$ represents the GP prior and $p_{\bm{\psi}}(\mathbf{w}) = \cN(\mathbf{w} \vert \bm{0}, \operatorname{diag}(\bm{\psi}))$ is a Gaussian prior over the neural network weights, assuming a Bayesian neural network with diagonal covariance.
%
%
Similarly, instead of using Eq.~\eqref{eq:variational_distribution_marginalized_general}, we employ the following variational distribution: 
\begin{equation} 
\label{eq:ETGPSSM_variational}
q(\x_{0}, \mathbf{w}, \tilde{f}) = q(\tilde{f}) q(\mathbf{w}) q(\x_0), 
\end{equation} 
which assumes independence between $\tilde{f}$, $\mathbf{w}$ and $\x_0$.  
Here the variational distributions of $\x_0$ and $\mathbf{w}$ are defined as follows:
\begin{equation}
    \left\{
    \begin{aligned}
        & q(\x_{0}) = \mathcal{N}(\x_0 \vert \m_{0}, \mathbf{L}_{0}\mathbf{L}_{0}^\top),\\
        & q(\mathbf{w}) = \mathcal{N}(\mathbf{w} \vert \m_{\mathrm{w}}, \operatorname{diag}(\bm{\sigma}^2_{\mathrm{w}})),\\
    \end{aligned}
    \right.
\end{equation}
where ${\m}_{0} \!\in\! \mathbb{R}^{d_x}$, lower-triangular matrix $\mathbf{L}_{0} \!\in\! \mathbb{R}^{d_x \times d_x}$, and $\{\mathbf{w}, \bm{\sigma}^2_{\mathrm{w}}\}$ are free variational parameters.  

For the variational distribution $q(\tilde{f})$, we adopt a sparse GP approximation \cite{titsias2009variational,hensman2013gaussian}, as defined in Eq.~\eqref{Eq:variational_GP_poserior}, which helps to reduce the computational complexity from $\mathcal{O}(T^3)$ to $\mathcal{O}(M^3)$.

With Eqs.~\eqref{eq:ETGPSSM_joint} and \eqref{eq:ETGPSSM_variational}, we can derive our ELBO, $\mathcal{L}$, which is summarized in the following proposition.

\begin{proposition}
		\label{prop:approx_ELBO_}
		Under all the assumed variational distributions, together with the model joint distribution in Eq.~\eqref{eq:ETGPSSM_joint}, the ELBO can be reformulated as follows:
        \begin{align}
             \mathcal{L} = &  \mathbb{E}_{q(\mathbf{w}, \tilde{f})}\left[ \sum_{t=1}^T \log p(\y_t \vert \y_{1:t-1}, \mathbf{w}, \tilde{f}) \right] -\operatorname{KL}(q(\u) \| p(\u)) \nonumber \\
                & -\operatorname{KL}(q(\mathbf{w}) \| p(\mathbf{w})) -\operatorname{KL}(q(\x_0) \| p(\x_0)), 
                \label{eq:EnVI_lower_bound}
        \end{align}
		where the first terms (log-likelihood) can be analytically evaluated using the EnKF (detailed in Section~\ref{subsec:EnK_ELBO_evaluation}). The three KL divergence terms can also be computed in closed form, due to the Gaussian nature of the prior and variational distributions \cite{theodoridis2020machine}.
\end{proposition}
\begin{proof}
    The proof can be found in Appendix~\ref{appex:ELBO_derivation}.
\end{proof}
\begin{remark}
The objective in Eq.~\eqref{eq:EnVI_lower_bound} can also be derived from $\mathcal{L}_2$ under two specific assumptions (see Supplement~\textcolor{blue}{A}), akin to \cite{lin2023ensemble}. However, our derivation in this work is generic for arbitrary random processes and is assumption-free, offering greater simplicity, conceptual clarity, and broader applicability. For further details on the relationship between $\mathcal{L}$ and $\mathcal{L}_2$,  we refer the reader to Supplement~\textcolor{blue}{A}.
\end{remark}

\subsection{EnKF-aided ELBO Evaluation}
\label{subsec:EnK_ELBO_evaluation}

We next describe how to evaluate $p(\y_t \vert \y_{1:t-1}, \mathbf{w}, \tilde{f})$ in Eq.~\eqref{eq:EnVI_lower_bound} using the EnKF. First, note that
\begin{equation}
\label{eq:evaluation_log_likelihood}
    \begin{aligned}
        p(\y_t \vert \y_{1:t-1}, \mathbf{w}, \tilde{f}) & =  \int \underbracket{p(\y_t \vert \x_t)}_{\text{likelihood}}  \underbracket{p(\x_t \vert \mathbf{w}, \tilde{f}, \y_{1:t-1})}_{\text{prediction}} \mathrm{d} \x_{t},
    \end{aligned}
\end{equation}
where the prediction distribution $p(\x_t \vert \mathbf{w}, \tilde{f}, \y_{1:t-1})$ reads:
\begin{equation}
\label{eq:prediction_etgpssm}
     \int \underbracket{p(\x_t \vert \mathbf{w}, \tilde{f}, \x_{t-1})}_{\text{transition}} \underbracket{p(\x_{t-1} \vert \mathbf{w}, \tilde{f}, \y_{1:t-1})}_{\text{filtering}}
 \mathrm{d} \x_{t-1}.
\end{equation}
The transition term can be expressed as
\begin{align}
    p(\x_t \vert  \mathbf{w}, \tilde{f}, \x_{t-1}) & = \int p(\x_t \vert \f_t) p(\f_t \vert \mathbf{w}, \tilde{f}, \x_{t-1}) \mathrm{d} \f_t \nonumber \\
    & = \cN\left(\x_{t} \mid \bm{\alpha}_t \cdot \tilde{f}(\x_{t-1}) + \bm{\beta}_t, \ \mathbf{Q} \right), \label{eq:etgp_transition}
\end{align}
with $p(\f_t \vert \mathbf{w}, \tilde{f}, \x_{t-1}) = \delta\left( \f_t - (\bm{\alpha}_t \cdot \tilde{f}(\x_{t-1}) + \bm{\beta}_t) \right)$, and $\{\bm{\alpha}_t, \bm{\beta}_t\}$ being the outputs of the neural network parameterized by $\mathbf{w}$.

Consequently, we employ the EnKF, a widely used method for handling nonlinearities and high-dimensional state spaces \cite{evensen1994sequential}, to estimate the filtering distribution $p(\x_{t} \vert \mathbf{w}, \tilde{f}, \y_{1:t})$ at each time step $t$. The details are as follows.

\subsubsection{\underline{\textbf{Prediction}}} Given the ensemble of states $\{\x_{t-1}^{(i)}\}_{i=1}^N$ from the posterior distribution  $p(\x_{t-1} \vert \mathbf{w}, \tilde{f}, \y_{1:t-1})$ at time $t - 1$, and conditioning on $\mathbf{w} \sim q(\mathbf{w})$ and $\tilde{f} \sim q(\tilde{f})$, we use the state transition in Eq.~\eqref{eq:etgp_transition} to generate $N$ predictive samples:
\begin{equation}
\label{eq:ETGP_predictive}
 \!\! \bar{\x}_{t}^{(i)} \!\sim\!  \cN\left(\x_{t} \mid \bm{\alpha}_t \cdot \tilde{f}(\x_{t-1}^{(i)}) + \bm{\beta}_t, \ \mathbf{Q} \right) 
\end{equation}
for $1 \leq i \leq N$.
The prediction distribution in Eq.~\eqref{eq:prediction_etgpssm} is then approximated by the predictive ensemble:
\begin{equation}
\label{eq:predictive_distri_enkf}
    p(\x_t \vert \mathbf{w}, \tilde{f}, \y_{1:t-1}) \approx \cN( \x_t \mid \bar{\mathbf{m}}_t, \ \bar{\mathbf{P}}_t),
\end{equation}
where
\begin{subequations}
\label{eq:empirical_mean_cov}
\begin{align}
    \overline{\mathbf{m}}_t & =\frac{1}{N} \sum_{i=1}^N \overline{\mathbf{x}}_t^{(i)}, \\
    \overline{\mathbf{P}}_t & =\frac{1}{N-1} \sum_{i=1}^N\left(\overline{\mathbf{x}}_t^{(i)}-\overline{\mathbf{m}}_t\right)\left(\overline{\mathbf{x}}_t^{(i)}-\overline{\mathbf{m}}_t\right)^{\top}.
\end{align}
\end{subequations}

\subsubsection{\underline{\textbf{Update}}}
Observations $\mathbf{y}_t$ is incorporated to update the predictive ensemble. Each sample is updated via:
\begin{subequations}
\label{eq:EnKF_update}
    \begin{align}
        & \ \ \ \ \ \ \text{\textit{update}:}  \ \ \ \ \mathbf{x}_t^{(i)} = \bar{\mathbf{x}}_t^{(i)} + \mathbf{G}_t (\mathbf{y}_t + \mathbf{e}_t^{(i)} - \bm{C} \bar{\mathbf{x}}_t^{(i)}), \\
        & \text{\textit{reparameterization}:}  \ \ \mathbf{e}_t^{(i)} = \bm{0} + \mathbf{R}^{\frac{1}{2}} \bm{\epsilon}, \   \bm{\epsilon} \!\sim\! \cN(\bm{0}, \bm{I}_{d_y}),    
        \label{eq:repara_2}
    \end{align}
\end{subequations}
where $\mathbf{G}_t$ is the Kalman gain matrix given by:
\begin{equation}
    \mathbf{G}_t = \bar{\mathbf{P}}_t \bm{C}^\top (\bm{C} \bar{\mathbf{P}}_t \bm{C}^\top + \mathbf{R})^{-1}. 
\end{equation}
This yields the updated ensemble $\{\x_{t}^{(i)}\}_{i=1}^N$ which is used to recursively propagate forward through time.
\begin{remark}
    The prediction and update steps are fully differentiable, as the sampled latent states become differentiable with respect to the state transition function $\bm{f}(\cdot)$ and the noise covariances $\mathbf{Q}$ and $\mathbf{R}$. This allows for end-to-end gradient-based optimization (e.g.using Adam \cite{kingma2015adam}, as described in Algorithm~\ref{alg:EnKF_version2}), making the framework practical for joint learning and inference \cite{chen2022autodifferentiable, lin2023ensemble}.
\end{remark}

Based on the outlined prediction and update steps, the likelihood $p(\y_t \vert  \y_{1:t-1}, \mathbf{w}, \tilde{f})$ in Eq.~\eqref{eq:evaluation_log_likelihood} can be easily evaluated. At each time step $t$, we have:
\begin{equation} \label{eq:evaluation_log_likelihood_analytical}
    \begin{aligned}
      p(\y_t \vert  \y_{1:t-1}, \mathbf{w}, \tilde{f}) \approx \cN(\y_t \mid  \bm{C} \bar{\mathbf{m}}_t, \ \bm{C} \bar{\mathbf{P}}_t \bm{C}^\top ),
    \end{aligned}
\end{equation}
due to the Gaussian prediction distribution, see Eq.~\eqref{eq:predictive_distri_enkf}, and the linear emission model, see Eq.~\eqref{eq:linear_emission}.

We now proceed to evaluate the variational lower bound $\mathcal{L}$, as defined in Eq.~\eqref{eq:EnVI_lower_bound}. \cblue{At each iteration, we apply the reparameterization trick to sample from the Gaussian variational distributions $q(\mathbf{w})$, $q(\mathbf{x}_0)$, and $q(\tilde{f})$. For instance, a sample from \( q(\tilde{f}) \) at input \( \mathbf{x}_{t} \) is generated as
\begin{align*}
    \tilde{f}(\mathbf{x}_{t}) = \xi_p(\mathbf{x}_{t}) + \sqrt{\Xi_p(\mathbf{x}_{t})} \cdot \epsilon, \quad \epsilon \sim \mathcal{N}(0, \, 1),
\end{align*}
where $\xi_p(\cdot)$ and $\Xi_p(\cdot)$ are given by Eq.~\eqref{eq:variational_GP_poserior_details}. Similarly, for $q(\mathbf{w})$ and $q(\mathbf{x}_0)$, we also use the standard Gaussian reparameterization.}
This enables unbiased Monte Carlo estimates of the expected log-likelihood term in $\mathcal{L}$.
And thanks to the reparameterization trick \cite{kingma2019introduction}, the ELBO $\mathcal{L}$ remains differentiable with respect to both the model parameters $\bm{\vartheta} \!=\! \{\bm{\theta}_{gp}, \bm{\psi}, \mathbf{Q}, \mathbf{R} \}$ and the variational parameters $\bzeta \!=\! \{ \mathbf{m}_0, \mathbf{L}_0, \mathbf{m}_{\mathrm{w}}, \bm{\sigma}_{\mathrm{w}}^2, \mathbf{m}, \mathbf{S}, \mathbf{Z} \}$. We leverage modern automatic differentiation tools such as PyTorch to compute gradients via backpropagation through time, and apply gradient-based optimization methods (e.g., Adam \cite{kingma2015adam}) to maximize $\mathcal{L}$ \cite{paszke2019pytorch}. The complete procedure for training ETGPSSM with EnKF-based latent state inference is summarized in Algorithm \ref{alg:EnKF_version2}.

\begin{algorithm}[t!]
    \caption{ETGPSSM}
    \label{alg:EnKF_version2}
    \begin{algorithmic}[1]
        \Statex {\bf Input}:  $\bm{\vartheta} = \{\btheta_{gp}, \bm{\psi}, \mathbf{Q}, \mathbf{R}\}, ~\bzeta, ~\y_{1:T}$, $\x_0^{1:N} \sim  q(\x_0)$
        \While{\textit{iterations not terminated}}
        \State $\mathbf{w} \sim q(\mathbf{w})$, $L_\ell = 0$
        \For {$t= 1, 2, \ldots, T$}
        \State \cblue{Get $\tilde{f}(\x_{t-1}^{(i)})$ using Eq.~\eqref{Eq:variational_GP_poserior} with $\x_* = \x_{t-1}^{(i)}, \forall i$}
        \Statex \qquad\quad \hrulefill \textbf{ \textit{EnKF} }\hrulefill
        \State Get prediction samples using Eq.~\eqref{eq:ETGP_predictive}
        \State Get empirical moments $\bar{\mathbf{m}}_t, \bar{\mathbf{P}}_t$ using Eq.~\eqref{eq:empirical_mean_cov}
        \State Get Kalman gain: $\bar{\mathbf{G}}_t \!=\! \bar{\mathbf{P}}_t \bm{C}^{\top} (\bm{C} \bar{\mathbf{P}}_t \bm{C}^{\top} \!\!+\! \mathbf{R})^{-1}$  
        \State  Get updated samples using Eq.~\eqref{eq:EnKF_update}
        \State Evaluate the log-likelihood using Eq.~\eqref{eq:evaluation_log_likelihood_analytical}, and 
        $$
        L_\ell \gets  L_\ell + \log p(\y_t \vert \y_{1:t-1}, \mathbf{w}, \tilde{f}) 
        \vspace{-.1in}
        $$  
        \Statex \qquad\quad \hrulefill 
        \EndFor
        \State Evaluate $\mathcal{L}$ based on $L_\ell$ and Eq.~\eqref{eq:EnVI_lower_bound} 
        \State Maximize $\mathcal{L}$ and update $\bm{\vartheta}$, $\bzeta$ using Adam \cite{kingma2015adam}
        \EndWhile
        \Statex {\bf Output}: EnKF particles $\{\x_{0:T}^{(i)}\}_{i=1}^N$,  model parameters $\bm{\vartheta}$, and variational parameters $\bzeta$.
    \end{algorithmic}
\end{algorithm}

\begin{remark} 
\label{remark:efficiency}  
    The proposed ETGPSSM in Algorithm~\ref{alg:EnKF_version2} improves upon existing GPSSMs in several key aspects of efficiency. First, it reduces the computational complexity from $\mathcal{O}(d_x M^3)$ to $\mathcal{O}(M^3)$ by relying on a single GP, while the added cost of a shallow neural network remains negligible. Second, by sharing the GP and flow network across dimensions, it avoids the quadratic scaling of variational parameters associated with independent GPs, as discussed in Section~\ref{subsec:challenges_GPSSMs}. Furthermore, by replacing explicit variational parameterization of the latent states with EnKF-based inference, the method significantly reduces parametric overhead. \cblue{Together, these design choices lead to substantial gain in both computational and modeling efficiency. The overall computational complexity is $\mathcal{O}(M^3 + d_x^3)$, where the $\mathcal{O}(d_x^3)$ term comes from the EnKF operations.

    It is important to note that in practice, $M$ is typically selected based on the complexity of the transition dynamics and the available computational budget, resulting in a trade-off between computation and modeling accuracy. Nevertheless, further investigation into how $M$ should scale with $d_x$ to preserve approximation accuracy in high-dimensional GPSSMs \cite{burt2020convergence} is warranted for future research.} 
\end{remark}



\begin{figure*}[t!]
    \centering

    \subfloat{\includegraphics[width=.99\textwidth]{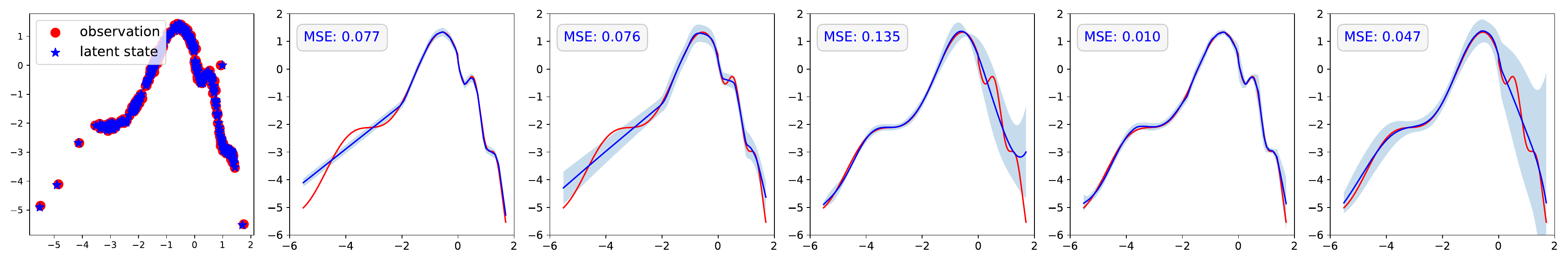} \label{subfig:variance_0}}
    
    \subfloat{\includegraphics[width=.99\textwidth]{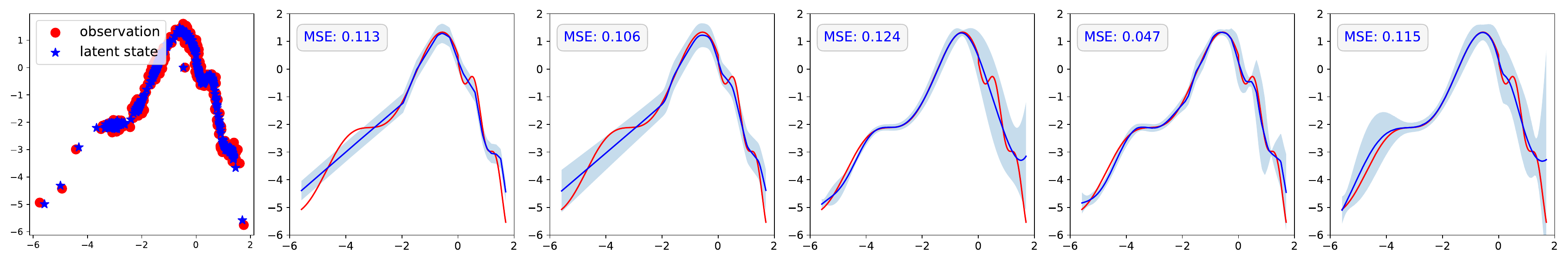} \label{subfig:variance_1}} 
    
    \subfloat{\includegraphics[width=.99\textwidth]{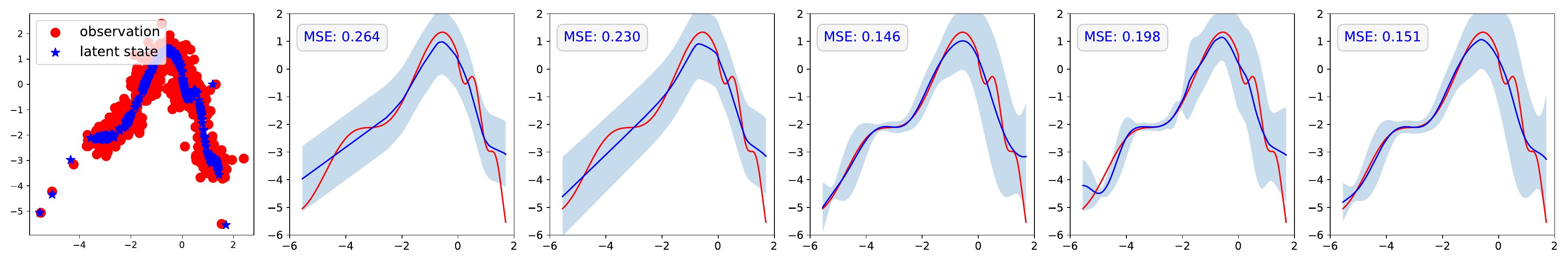} \label{subfig:variance_2}} 

    \begin{subfigure}{0.99\textwidth}
        \includegraphics[width=\textwidth]{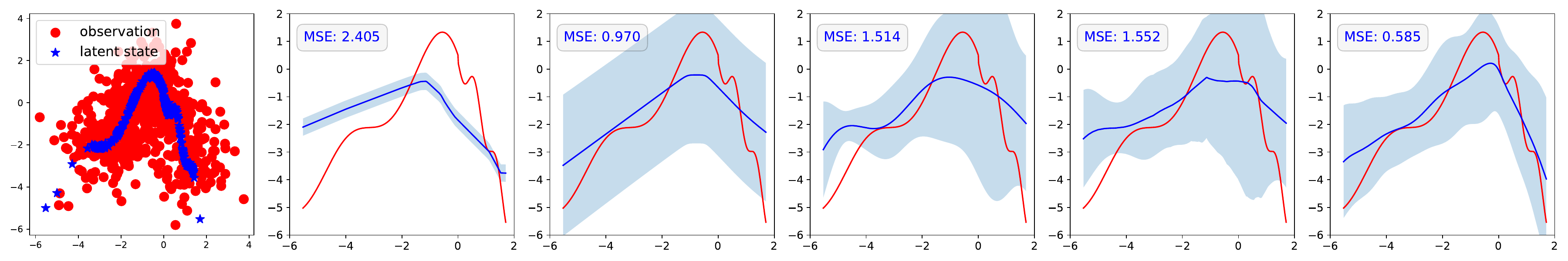} 
        \caption*{
        \hspace{.5in} Data  
        \hspace{.58in} AD-EnKF \cite{chen2022autodifferentiable}
        \hfill AD-EnKF \textsc{(bnn)}
        \hspace{.35in} EnVI \cite{lin2023ensemble}
        \hfill ETGPSSM \textsc{(dnn)}
        \hspace{.18in} ETGPSSM \textsc{(bnn)}
        \hfill
        }
        \label{subfig:variance_3}
    \end{subfigure}
    
    \caption{Non-stationary \textit{kink} transition function learning performance (mean $\pm$ $2 \sigma$) using various methods across different levels of emission noise ($\sigma_{\mathrm{R}}^2 \in \{0.0008, 0.008, 0.08, 0.8\}$, from top to bottom). The blue curve (\textcolor{blue}{\textbf{---}}) represents the learned mean function, while the red line (\cred{\textbf{---}}) indicates the true system transition function. \cblue{The shaded region depicts the total predictive uncertainty for the state transition, which includes the estimated transition noise level $(\pm 2 \, \hat{\sigma}_Q)$ for all methods, plus function uncertainty for probabilistic models (BNN/GP).}
    }
    \label{fig:kink_EnVI}
\end{figure*}

\section{Experiments} 
\label{sec:experiments}
In this section, we empirically evaluate the performance of ETGPSSM across a range of tasks. We begin by evaluating ETGPSSM’s ability to model non-stationary dynamical systems in Section~\ref{subsec:non_stationary_DSL}. We then assess its computational and parametric efficiency, as well as filtering accuracy, in chaotic high-dimensional settings in Section~\ref{subsec:Lorenz96_system}. Lastly, Section~\ref{subsec:time_series_prediction} highlights the model’s predictive performance across a range of real-world time series datasets.
Unless otherwise stated, all GPSSMs use the SE kernel \cite{williams2006gaussian}, with $20$ inducing points initialized at random. The shallow neural network used in ETGPSSM is a fully connected three-layer architecture with ReLU activations and hidden dimensions $128$ and $64$, mapping from $\mathbb{R}^{d_x}$ to the flow parameter space $\Theta \! \subseteq \! \mathbb{R}^{2d_x}$. \cblue{For further implementation details, we refer the reader to the supplementary material} and the public code repository.\footnote{\url{https://github.com/zhidilin/gpssmProj/tree/main/high_dim_GPSSM}} 

\subsection{Non-Stationary Dynamical System Learning}
\label{subsec:non_stationary_DSL}
This subsection evaluates the modeling capabilities of ETGPSSM on a non-stationary dynamical system, comparing its performance to state-of-the-art GPSSM and neural network-based methods. The underlying system is defined as:
\begin{subequations}
\label{eq:kink_function_dynamics}
    \begin{align}
        & \x_{t+1}  = f(\x_t) + \mathbf{v}_{t}, & \mathbf{v}_{t} \sim \cN(0, \sigma_{\mathrm{Q}}^2), \\
        & \y_t = \x_t + \mathbf{e}_t, & \mathbf{e}_t\sim \cN(0, \sigma_{\mathrm{R}}^2),
    \end{align}
\end{subequations}
where the \textit{non-stationary kink function} $f(\x)$ is defined as:
\begin{equation}
    f(\x) \!=\! \underbrace{\left[0.8 + (\x + 0.2) \left(1 - \frac{5}{1 + \exp(-2\x)}\right) \right]}_{\text{kink function}} s(\x) - o(\x)
\end{equation}
with a slope modulation:
\begin{equation}
s(\x) = 
\begin{cases}
    1 - 0.5 \exp(-0.5\x), & \x > 0, \\
    1, & \x \leq 0,
\end{cases}
\end{equation}
and an oscillatory component:
\begin{equation}
o(\x) = 
\begin{cases}
    0.5 \sin(8\x), & \x > 0, \\
    0.5 \sin(2\x), & \x \leq 0.
\end{cases}
\end{equation}
The kink function is a standard benchmark for evaluating GPSSMs \cite{frigola2014variational, ialongo2019overcoming, lin2023ensemble}, \cblue{and is readily modeled by a stationary GP with an SE kernel. Here, we break this stationarity by extending it with input-dependent slope and oscillatory terms, creating a non-stationary benchmark that presents a greater challenge for dynamical models (see Fig.~\ref{fig:kink_EnVI}).} 
We generate training data $\y_{1:T}$ by fixing $\sigma_\mathrm{Q}^2$ at $0.05$ and varying $\sigma_\mathrm{R}^2$ systematically across the set $\{0.0008,0.008,0.08,0.8\}$, producing four noise regimes. ETGPSSM is compared to two leading baselines: the neural network-based AD-EnKF \cite{chen2022autodifferentiable} and the GP-based EnVI \cite{lin2023ensemble}. \cblue{The AD-EnKF employs a fully-connected network with ReLU activations and hidden layers of dimensions 128 and 64, to model the state transition. Further implementation details (e.g., learning rate, optimizer) are provided in Supplement~\textcolor{blue}{B}.}
Consistent with previous works \cite{ialongo2019overcoming,lin2023ensemble}, all methods share the true emission model while learning the transition function, ensuring a fair comparison. The system dynamics learning results across different methods are shown in Fig.~\ref{fig:kink_EnVI}, where our method is represented in the last two columns. The suffixes ``\textsc{dnn}'' and ``\textsc{bnn}'' in brackets  indicate the network types—Bayesian neural network (BNN) or deep neural network (DNN)—used in our model (see Fig.~\ref{fig:ETGP}).

\begin{figure*}[t!]
    \centering
    \includegraphics[width=.97\linewidth]{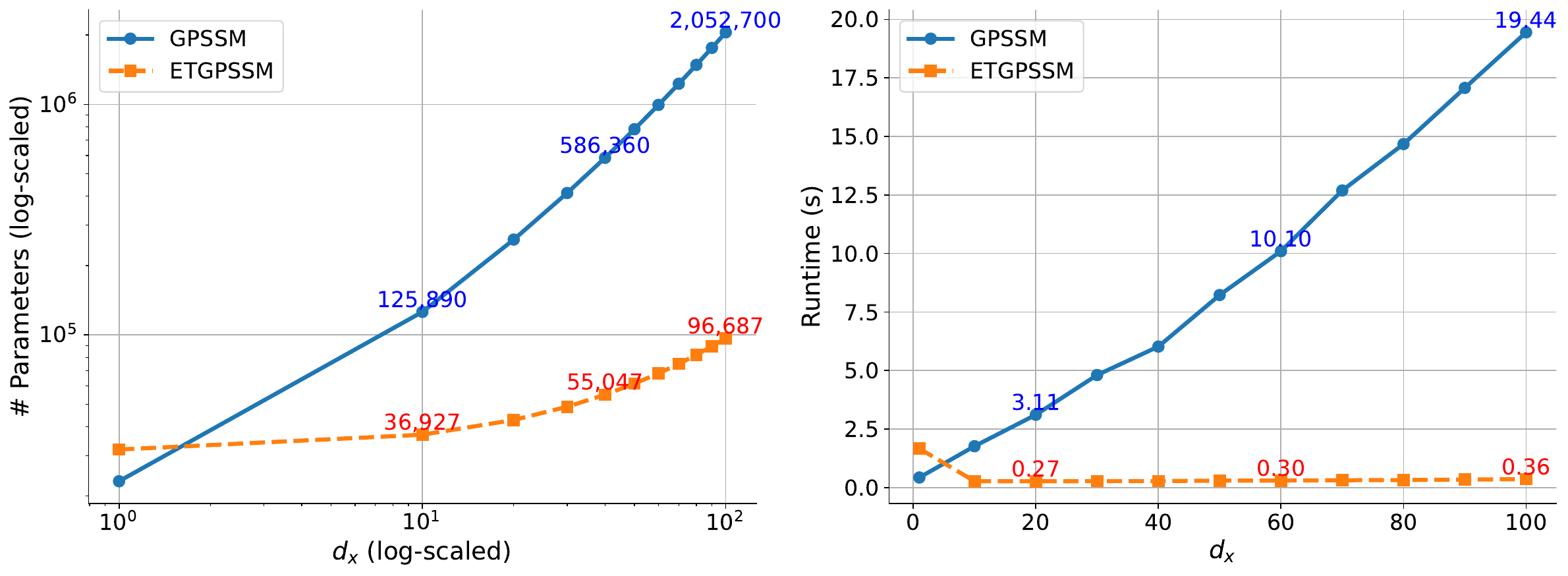}
    \caption{(\textbf{Left}) \cblue{Comparison of the total number of parameters (including variational and model parameters). In ETGPSSM, the count is dominated by the GP ($M d_x + 3 + M + M^2$) and the neural network with architecture $d_x \!\to\! 128 \!\to\! 64 \!\to\! 2d_x$ ($258\,d_x + 8384$), resulting in overall linear scaling in $d_x$. In GPSSM, the parameters are dominated by the $d_x$ independent GPs, ($M d_x^2 + (3 + M + M^2) d_x$), leading to quadratic scaling in $d_x$.} (\textbf{Right}) ETGPSSM maintains low computational costs as $d_x$ increases, while GPSSM exhibits linear growth.
    }
    \label{fig:parameters_time_comparison}
\end{figure*}

The results in Fig.~\ref{fig:kink_EnVI} show that, across all noise levels, ETGPSSM achieves consistently superior system identification performance. In the low-noise regimes($\sigma_{\mathrm{R}}^2 = \{0.0008, 0.008\}$), the DNN-based ETGPSSM excels at capturing both smooth and oscillatory regions of the transition function. By contrast, AD-EnKF struggles with the smooth left segment and performs well only on the right, highly nonlinear region. EnVI, the state-of-the-art GPSSM baseline, fails to capture the non-stationarity and underfits the high-frequency component.

Under high noise $(\sigma_{\mathrm{R}}^2 = \{0.08, 0.8\})$, performance declines across all methods due to increased latent state uncertainty. Nevertheless, ETGPSSM remains robust. Notably, in this scenario, BNN-based methods, particularly ETGPSSM (\textsc{bnn}), outperforms the DNN-based version, benefiting from regularization and better uncertainty quantification. These results emphasize the importance of modeling flexibility and demonstrate the advantage of combining neural flows with GP priors.

\subsection{High-Dimensional Latent State Estimation}
\label{subsec:Lorenz96_system}

In this subsection, we evaluate latent state estimation in the Lorenz-96 system, a standard benchmark for studying chaotic dynamics in weather and climate modeling \cite{lorenz1996predictability}. The system consists of \( d_x \in \mathbb{Z}_{+} \) coupled ordinary differential equations (ODEs):
\begin{equation}
    \frac{d\mathrm{x}_{d}}{dt} = (\mathrm{x}_{d+1} - \mathrm{x}_{d-2}) \mathrm{x}_{d-1} - \mathrm{x}_{d} + F, \ 1\le d\le d_x
\end{equation}
where $F=8$ induces fully chaotic behavior, with small differences in initial conditions leading to divergent trajectories \cite{lorenz1996predictability}. The dynamics model is discretized using the Euler method, and at each time step $t$, observations are generated through a linear emission model ($\bm{C} \!=\! \bm{I}$) with Gaussian additive noise, where the noise level is significant with the covariance given by $\mathbf{R} = 4 \bm{I}$.

We first assess the scalability of ETGPSSM in high-dimensional settings. \cblue{For simplicity, we focus exclusively on the \textsc{dnn} variant of ETGPSSM. The \textsc{bnn} variant achieves comparable performance but doubles the number of parameters by introducing variance parameters for each neural network weight. This observation aligns with prior work showing that, on sufficiently large datasets, the \textsc{bnn} variant offers no significant performance gain over its deterministic counterpart \cite{maronas2022efficient,maronas2021transforming}.} Figure~\ref{fig:parameters_time_comparison} shows the parameter count and runtime of ETGPSSM and standard GPSSMs, both using 100 inducing points. \cblue{As shown in the left panel, ETGPSSM exhibits linear growth in parameter count with increasing $d_x$, while GPSSM grows quadratically.} This is enabled by ETGPSSM’s use of a shared GP combined with normalizing flows, in contrast to GPSSM's reliance on independent GPs per dimension.

The right panel of Fig.~\ref{fig:parameters_time_comparison} illustrates the computational efficiency of ETGPSSM by comparing the one-sample execution time for transition evaluation (i.e., Eq.~\eqref{eq:ssm_transition}). Notably, ETGPSSM maintains consistently low running times across increasing state dimensions, whereas GPSSM exhibits rapid growth in computational cost. For instance, at $d_{x} = 100$, ETGPSSM completes the evaluation in $0.36$ seconds, significantly outperforming GPSSM, which requires $19.56$ seconds. This efficiency stems from the reduced computational complexity of ETGPSSM, which scales as $ \mathcal{O}(M^{3})$ due to its use of a single GP, in contrast to GPSSM’s $\mathcal{O}(d_{x}M^{3})$ complexity. These results demonstrate that ETGPSSM remains computationally and parametrically tractable in high-dimensional latent spaces.

\begin{figure*}[t!]
    \centering
    \includegraphics[width=0.99\linewidth]{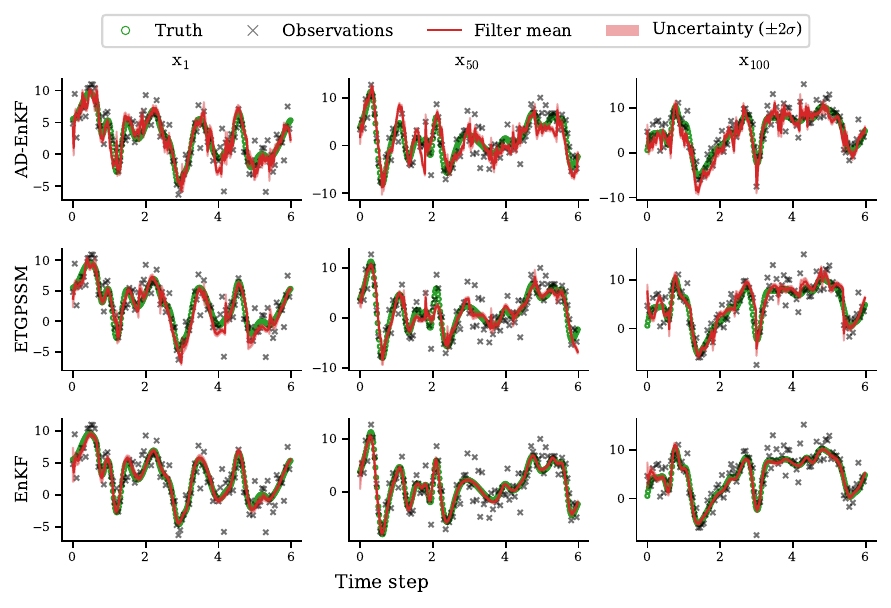}
    \vspace{-.1in}
    \caption{\cblue{Latent state estimation performance for a 100-dimensional Lorenz-96 system, with 3 dimensions sampled for visualization.  Root mean square error (RMSE) values for different methods are 0.4468 for EnKF, 1.1583 for ETGPSSM, and 1.4553 for AD-EnKF. RMSE between the true latent states and the observations is 1.9855, which serves as a baseline, indicating the error if observations were used directly as state estimates.}
    }
    \label{fig:filtering_performance}
\end{figure*}

\begin{table}[t!]
\centering
\caption{\cblue{Comparison of filtering performance in the pendulum system.}}
\cblue{
\adjustbox{valign=c, width=0.98\linewidth}{
    \begin{tabular}{lcccccc}
    \toprule
    Method & \textsc{rmse ($\downarrow$)}  & \textsc{coverage} & \textsc{spread} & \textsc{crps ($\downarrow$)}  \\
    \midrule \midrule
    \textsc{ad-enkf}  & 1.4553 & 0.3007 & 0.2885 &  0.2681 \\
    \textsc{etgpssm} & 1.1583 & 0.3587 & 0.2774 &  0.2101 \\
    \textsc{enkf}     & 0.4468 & 0.5866 & 0.1741 &  0.0685 \\
    \bottomrule
    \end{tabular}
    }
}
\label{tab:Lorenz_96}
\end{table}

We next fix the dimension $d_{x} \!=\! 100$ in the Lorenz-96 system and generate observations of length $T \!=\! 600$ for training ETGPSSM and its deep neural network-based competitor, AD-EnKF \cite{chen2022autodifferentiable}. The EnVI \cite{lin2023ensemble}, which already caused out-of-memory errors\footnote{All experiments were conducted on a machine equipped with an Intel Xeon W7-3545 v6 processor (24 cores, 2.7 GHz), 128 GB of DDR5 RAM, and an NVIDIA RTX 4500 Ada GPU (24 GB VRAM), running Ubuntu 20.04 LTS.} at $d_x \!=\! 30$ with an ensemble size $N \!=\! 150$, is excluded due to its prohibitive computational cost in this high-dimensional setting.  The corresponding state inference performance after training is summarized in Fig.~\ref{fig:filtering_performance} \cblue{and Table~\ref{tab:Lorenz_96}. Table~\ref{tab:Lorenz_96} presents four standard filtering metrics: RMSE, spread, coverage probability, and CRPS \cite{spantini2022coupling,gneiting2007probabilistic}. Detailed definitions of these metrics are provided in Supplement~\textcolor{blue}{B}.} For comparative analysis, the performance of EnKF is also included.

\cblue{As shown in Fig.~\ref{fig:filtering_performance} and Table~\ref{tab:Lorenz_96}, ETGPSSM achieves superior performance compared to AD-EnKF in both mean estimation (see e.g. \textsc{RMSE}) and uncertainty quantification (see e.g. \textsc{CRPS}). Notably, with the observation RMSE of 1.9855 serving as a baseline, ETGPSSM demonstrates substantial error reduction relative to the raw observational data.} This result highlights the superior filtering performance of ETGPSSM in high-dimensional chaotic systems.
While EnKF achieves the best filtering performance, e.g., lowest RMSE, it assumes full knowledge of the dynamics and performs no learning. In contrast, ETGPSSM simultaneously learns the system and infers latent states, offering a more general and powerful solution for complex, high-dimensional systems.

\subsection{Real-World Time Series Prediction}
\label{subsec:time_series_prediction}
In this subsection, we evaluate the prediction performance of ETGPSSM on several real-world time series datasets\footnote{\url{https://homes.esat.kuleuven.be/~smc/daisy/daisydata.html}} commonly used for system identification (see Table~\ref{tab:systemidentifcation}). 
In every dataset, the first half of the sequence is utilized as training data, with the remaining portion designated for testing. Standardization of all datasets is conducted based on the training sequence. Table~\ref{tab:systemidentifcation} reports the series prediction results, wherein the RMSE is averaged over 50-step ahead forecasting.
We compare ETGPSSM against state-of-the-art methods, including neural network-based approaches—DKF \cite{krishnan2017structured} and AD-EnKF \cite{chen2022autodifferentiable}—as well as various GPSSMs: vGPSSM \cite{eleftheriadis2017identification}, CO-GPSSM \cite{lin2023towards}, PRSSM \cite{doerr2018probabilistic}, ODGPSSM \cite{lin2022output}, VCDT \cite{ialongo2019overcoming}, and EnVI \cite{lin2023ensemble}. Following previous works, the latent space dimension is set to $d_x = 4$ for all the methods. 

\begin{table*}[ht!]
    \centering
    \caption{The 50-step ahead prediction performance in terms of RMSE of different models (all with $d_x = 4$) on system identification datasets.  Mean and standard deviation of the prediction results are shown as an averaged result over ten different seeds. The three lowest RMSE values are highlighted in shades of green, with a deeper one indicating lower RMSE.} 
    \setlength{\tabcolsep}{4.2mm}{
        \centering
        \begin{tabular}{r | ccccc }
            \toprule
            Method & \multicolumn{1}{c}{Actuator} &  \multicolumn{1}{c}{Ball Beam} &  \multicolumn{1}{c}{Drive} &  \multicolumn{1}{c}{Dryer} &  \multicolumn{1}{c}{Gas Furnace}\\
            \midrule \midrule
            \textbf{DKF} \cite{krishnan2017structured}
            & $1.204 \pm 0.250$   
            & $0.144 \pm 0.005$  
            & $0.735 \pm 0.001$   
            & $1.465 \pm 0.087$   
            & $5.589 \pm 0.066$    \\
            \textbf{AD-EnKF} \cite{chen2022autodifferentiable}
            &  $0.705 \pm 0.117$  
            &  $0.057 \pm 0.006$
            &  $0.756 \pm 0.114$ 
            &  $0.182 \pm 0.053$ 
            &  $1.408 \pm 0.090$\\ 
            \textbf{AD-EnKF} (\textsc{bnn})
            & $0.705 \pm 0.088$
            & \cellcolor[rgb]{0.635, 0.843, 0.694}$0.053 \pm 0.007$
            & $0.896 \pm 0.088$
            & $0.155 \pm 0.030$
            & \cellcolor[rgb]{0.847, 0.929, 0.875}$1.361 \pm 0.061$\\
            \midrule \midrule
            \textbf{vGPSSM}  \cite{eleftheriadis2017identification}
            &  $1.640 \pm 0.011$ 
            &  $0.268 \pm 0.414$
            &  $0.740 \pm 0.010$
            &  $0.822 \pm 0.002$
            &  $3.676 \pm 0.145$   \\
            \textbf{CO-GPSSM} \cite{lin2023towards}
            &  $0.803 \pm 0.011$
            &  $0.079 \pm 0.018$
            &  $0.736 \pm 0.007$
            &  $0.366 \pm 0.146$
            &  $1.898 \pm 0.157$ \\
            \textbf{PRSSM}  \cite{doerr2018probabilistic}
            &  $0.691 \pm 0.148$
            &  $0.074 \pm 0.010$
            &  \cellcolor[rgb]{0.388, 0.745, 0.482}$0.647 \pm 0.057$
            &  $0.174 \pm 0.013$  
            &  $1.503 \pm 0.196$   \\ 
            \textbf{ODGPSSM}  \cite{lin2022output}
            &  \cellcolor[rgb]{0.847, 0.929, 0.875}$0.666 \pm 0.074$
            &  $0.068 \pm 0.006$
            &  $0.708 \pm 0.052$ 
            &  $0.171 \pm 0.011$  
            &  $1.704 \pm 0.560$ \\
            \textbf{VCDT}  \cite{ialongo2019overcoming}
            &  $0.815 \pm 0.012$
            &  $0.065 \pm 0.005$
            &  $0.735 \pm 0.005$
            &  $0.667 \pm 0.266$
            &  $2.052 \pm 0.163$ \\
            \textbf{EnVI} \cite{lin2023ensemble}  
            &  \cellcolor[rgb]{0.635, 0.843, 0.694}$0.657 \pm 0.095$
            &  \cellcolor[rgb]{0.847, 0.929, 0.875}$0.055 \pm 0.002$
            &  \cellcolor[rgb]{0.847, 0.929, 0.875}$0.703 \pm 0.050$
            &  \cellcolor[rgb]{0.388, 0.745, 0.482}$0.125 \pm 0.017$
            &  $1.388 \pm 0.123$ \\
            \midrule \midrule
            \textbf{ETGPSSM} (\textsc{dnn})  
            & \cellcolor[rgb]{0.388, 0.745, 0.482}$0.646 \pm 0.081$
            & \cellcolor[rgb]{0.388, 0.745, 0.482}$0.050 \pm 0.003$
            & \cellcolor[rgb]{0.635, 0.843, 0.694}$0.668 \pm 0.092$
            & \cellcolor[rgb]{0.635, 0.843, 0.694}$0.137 \pm 0.030$
            & \cellcolor[rgb]{0.388, 0.745, 0.482}$1.300 \pm 0.052$\\
            \textbf{ETGPSSM} (\textsc{bnn})  
            &  \cellcolor[rgb]{0.388, 0.745, 0.482}$0.646 \pm 0.095$
            &  \cellcolor[rgb]{0.635, 0.843, 0.694}$0.053 \pm 0.002$
            &  \cellcolor[rgb]{0.847, 0.929, 0.875}$0.703 \pm 0.028$
            &  \cellcolor[rgb]{0.847, 0.929, 0.875}$0.154 \pm 0.026$
            &  \cellcolor[rgb]{0.635, 0.843, 0.694}$1.313 \pm 0.048$\\
            \bottomrule
    \end{tabular}}
    \label{tab:systemidentifcation}
\end{table*}

As reported in Table~\ref{tab:systemidentifcation}, ETGPSSM demonstrates strong overall performance. On the \textit{Actuator} dataset, both ETGPSSM (\textsc{dnn}) and ETGPSSM (\textsc{bnn}) achieve an RMSE of $0.646$, outperforming EnVI ($0.657$) and significantly improving upon AD-EnKFs, which yield an RMSE of $0.705$. Similar trends are observed on the \textit{Ball Beam} and \textit{Gas Furnace} datasets, where ETGPSSM ranks first and second, respectively, indicating effective modeling of system dynamics. On the \textit{Drive} and \textit{Dryer} datasets, however, GPSSMs such as PRSSM and EnVI attain slightly better performance. This is likely due to the relatively low latent state dimension ($d_x = 4$), where conventional GPSSMs, including PRSSM, ODGPSSM, and EnVI, remain competitive, as discussed in \cite{lin2023ensemble}. Nonetheless, we highlight that although the primary contribution of the proposed ETGPSSM lies in its scalability and efficiency in high-dimensional settings, it also remains highly competitive in low-dimensional scenarios, achieving superior or very comparable performance.

In comparison to neural network-based methods like AD-EnKF, ETGPSSM’s hybrid design leverages the representational power of deep learning while retaining the uncertainty quantification and regularization benefits of GPs. This synergy proves especially effective for modeling complex, non-stationary dynamics. This is evident from the results, where neural network-based methods exhibit higher prediction errors. AD-EnKF (\textsc{bnn}) consistently outperforms its DNN counterpart, underscoring the robustness provided by Bayesian approaches. Still, both lag behind ETGPSSM in predictive accuracy across most datasets. In summary, ETGPSSM demonstrates strong predictive performance on diverse real-world time series, matching or outperforming prior methods, while offering a scalable and efficient solution applicable to both low- and high-dimensional dynamical systems.

\section{Conclusions} 
\label{sec:conclusions}
In this paper, we proposed a new framework for learning and inference in GPSSM tailored to non-stationary, high-dimensional dynamical systems.
Our central contribution is the introduction of the ETGP prior, which combines a shared GP with input-dependent normalizing flows parameterized by neural networks. This construction enables expressive, non-stationary modeling while reducing the computational and parametric overhead typical of high-dimensional GPSSMs.
To support scalable inference, we developed a variational framework that approximates the ETGP posterior by disentangling the GP and neural network components. The generic and assumption-free variational lower bound we derive requires no explicit latent state posterior parameterization, but instead uses EnKF for the inference, yielding additional computational efficiency.
Extensive experiments across system identification, high-dimensional filtering, and time series forecasting show that ETGPSSM consistently achieves strong performance while maintaining favorable computational efficiency.
%
\appendices

\section{Proof of Corollary~\ref{corollary:dependent_nonstationary_GP}}
\label{app:proof_corollary}
\begin{proof}
    Similar proof of this corollary has been discussed in the literature; see \cite{maronas2022efficient}. For ease of reference, we also detail here. 
    Consider two arbitrary time steps \(t\) and \(t^\prime\), we have 
    \begin{align}
        & \bm{f}(\x_{t-1}) = \bm{\alpha}_{t} \cdot \tilde{f}(\x_{t-1}) + \bm{\beta}_{t},\\
        & \bm{f}(\x_{t^\prime-1}) = \bm{\alpha}_{t^\prime} \cdot \tilde{f}(\x_{t^\prime-1}) + \bm{\beta}_{t^\prime}.
    \end{align}
    \cblue{Since the parameters $\bm{\theta}_t$ and $\bm{\theta}_{t^\prime}$ are deterministic when conditioned on the inputs $\mathbf{x}_t$ and $\mathbf{x}_{t^\prime}$, the joint distribution of \(\bm{f}(\x_{t-1})\) and \(\bm{f}(\x_{t^\prime-1})\) is Gaussian with mean $[\bm{\beta}_t^\top, \bm{\beta}_{t^\prime}^\top]^\top$.}
    The covariance between \(\bm{f}(\x_{t-1})\) and \(\bm{f}(\x_{t^\prime-1})\) is:
    \begin{align}
        \bm{\Lambda}_{t, t^\prime} & = \bm{\alpha}_{t} \cdot \text{Cov}(\tilde{f}(\x_{t-1}), \tilde{f}(\x_{t^\prime-1})) \cdot \bm{\alpha}_{t^\prime}^\top\\
        & = k(\x_{t-1}, \x_{t^\prime-1}) \cdot \bm{\alpha}_{t} \bm{\alpha}_{t^\prime}^\top.
    \end{align} 
    Similarly, the covariance of $\bm{f}(\x_{t-1})$ with itself is:
    \begin{align}
        \text{Cov}(\bm{f}(\x_{t-1}), \bm{f}(\x_{t-1})) = k(\x_{t-1}, \x_{t-1}) \cdot \bm{\alpha}_{t} \bm{\alpha}_{t}^\top.
    \end{align}
    The same logic applies to $\bm{f}(\x_{t^\prime-1})$.
    Therefore, the covariance matrix $\bm{\Lambda}_{t, t^\prime}$ is:
    \begin{equation} 
        \begin{bmatrix}
            k(\x_{t-1}, \x_{t-1}) \cdot \bm{\alpha}_t \bm{\alpha}_t^\top & k(\x_{t-1}, \x_{t^\prime-1}) \cdot \bm{\alpha}_t \bm{\alpha}_{t^\prime}^\top\\
            k(\x_{t^\prime-1}, \x_{t-1}) \cdot \bm{\alpha}_{t^\prime} \bm{\alpha}_t^\top & k(\x_{{t^\prime}-1}, \x_{{t^\prime}-1}) \cdot \bm{\alpha}_{t^\prime} \bm{\alpha}_{t^\prime}^\top
        \end{bmatrix}.
    \end{equation}
\end{proof}

\section{Proof of Proposition~\ref{prop:bound_comparison}}
\label{appx:bounds_comparision}
\begin{proof}
Our target is to compare the two lower bounds on the log-evidence $\log p(\y_{1:T})$:
\begin{align}
\mathcal{L} &=  \mathbb{E}_{q(\x_0, \bm{f})} \left[ \log \frac{p(\y_{1:T}, \x_0, \bm{f})}{q(\x_0, \bm{f})} \right],\\
\mathcal{L}_2 &= \mathbb{E}_{q(\x_{0:T}, \bm{f})} \left[ \log \frac{p(\y_{1:T}, \x_{0:T}, \bm{f})}{q(\x_{0:T}, \bm{f})} \right]. 
\end{align}
Since
$$
q(\x_{0:T}, \bm{f}) = q(\x_0, \bm{f})q(\x_{1:T}\vert \bm{f}, \x_0),
$$
we have $\mathcal{L}_2=$
\begin{equation}
\begin{aligned}
& \mathbb{E}_{q(\x_{0:T}, \bm{f})} \left[ \log \frac{p(\y_{1:T}, \x_{0}, \bm{f}) p(\x_{1:T} \vert \x_{0}, \bm{f}, \y_{1:T})}{q(\x_0, \bm{f})q(\x_{1:T}\vert \bm{f}, \x_0)} \right]  \\
&\!=\!  \mathbb{E}_{q(\x_{0:T}, \bm{f})} \left[ \log \frac{p(\y_{1:T}, \x_0, \bm{f})}{q(\x_0)q(\bm{f})} + \log \frac{p(\x_{1:T} \vert  \bm{f}, \x_0, \y_{1:T})}{q(\x_{1:T} \vert \bm{f}, \x_0)} \right]    \\
&\!=\! \mathcal{L} - \mathbb{E}_{q(\bm{f}, \x_0)} \left[ \text{KL} \big(q(\x_{1:T} \vert \bm{f}, \x_0) \| p(\x_{1:T}\vert \bm{f}, \x_0, \y_{1:T})\big) \right] 
\end{aligned}
\end{equation}
Since KL divergence is non-negative:
\begin{equation}
 \text{KL} \big(q(\x_{1:T} \vert \bm{f}, \x_0) \| p(\x_{1:T}\vert \bm{f}, \x_0, \y_{1:T})\big) \geq 0 \nonumber
\end{equation}
with equality iff $q(\x_{1:T} \vert \bm{f}, \x_0) = p(\x_{1:T}\vert \bm{f}, \x_0, \y_{1:T})$.
Thus we have:
\begin{equation}
\mathcal{L} \geq \mathcal{L}_2. \nonumber
\end{equation}
Equality holds when $q(\x_{1:T} \vert \bm{f}, \x_0) = p(\x_{1:T}\vert \bm{f}, \x_0, \y_{1:T})$ for all $\bm{f}$ and $\x_0$.
\end{proof}

\section{Proposition~\ref{prop:approx_ELBO_}: ELBO Derivation}
\label{appex:ELBO_derivation}
\begin{proof}
We derive the ELBO $\mathcal{L}$ for our ETGPSSM by first considering a general stochastic process prior $\bm{f}$ over the transition function (Eq.~\eqref{eq:ELBO_general}), then specializing to our ETGP construction. The derivation proceeds as follows: 
\begin{subequations}
\begin{align}
    \mathcal{L} & = \mathbb{E}_{q} \left[ \log \frac{p( \y_{1:T}, \x_{0}, \mathbf{w}, \tilde{f})}{q(\x_{0}, \mathbf{w}, \tilde{f})}\right], \\
    & = \mathbb{E}_{q} \left[ \log \frac{p(\y_{1:T} \vert \x_0, \mathbf{w}, \tilde{f}) p(\x_0) p(\mathbf{w}) p(\tilde{f})}{q(\x_0) q(\mathbf{w}) q(\tilde{f}) }\right]\\
    & = \mathbb{E}_{q} \left[ \log \frac{p(\y_{1:T} \vert \x_0, \mathbf{w}, \tilde{f}) p(\x_0) p(\mathbf{w}) p(\tilde{f} \vert \u) p(\u)}{q(\x_0) q(\mathbf{w}) p(\tilde{f} \vert \u) q(\u) }\right] \label{eq:sparse_GP_using_ELBO}\\
    & = \mathbb{E}_{q}\left[\log p(\y_{1:T} \vert \x_0, \mathbf{w}, \tilde{f})\right] - \operatorname{KL}(q(\u) \| p(\u)) \nonumber\\
    & \quad - \operatorname{KL}(q(\mathbf{w}) \| p(\mathbf{w}))- \operatorname{KL}(q(\x_0) \| p(\x_0)),\\
    & = \mathbb{E}_{q}\left[ \sum_{t=1}^T \log p(\y_{t} \vert \y_{1:t-1}, \mathbf{w}, \tilde{f})\right] - \operatorname{KL}(q(\u) \| p(\u))  \nonumber\\
    & \quad - \operatorname{KL}(q(\mathbf{w}) \| p(\mathbf{w}))- \operatorname{KL}(q(\x_0) \| p(\x_0)), \label{eq:summation_likelihood}
\end{align}
\end{subequations}
where we augment the GP using inducing points and use the variational sparse GP in Eq.~\eqref{eq:sparse_GP_using_ELBO}; We factorize the joint likelihood into sequential terms in Eq.~\eqref{eq:summation_likelihood}, where for $t=1$ we have the log-term as $\log p(\y_1 \vert \x_0, \mathbf{w}, \tilde{f})$.
\end{proof}

\bibliographystyle{IEEEtran}
\bibliography{main.bib}

@inproceedings{ma2019variational,
  title={Variational implicit processes},
  author={Ma, Chao and Li, Yingzhen and Hern{\'a}ndez-Lobato, Jos{\'e} Miguel},
  booktitle = {Proc. Int. Conf. Mach. Learn. (ICML)},
  pages={4222--4233},
  year={2019},
}

@article{suwandi2023sparsityaware,
        title={Sparsity-Aware Distributed Learning for {G}aussian Processes With Linear Multiple Kernel},
        author={Suwandi, Richard Cornelius and Lin, Zhidi and Yin, Feng and Wang, Zhiguo and Theodoridis, Sergios},
        journal   = {IEEE Trans. Neural Netw. Learn. Syst.},
        year={2025},
        month = jan,
        pages={1--15}
}

@PhdThesis{mchutchon2015nonlinear,
  author = {McHutchon, Andrew James},
  school = {University of Cambridge},
  title  = {{Nonlinear modelling and control using Gaussian processes}},
  year   = {2014},
}

@PhdThesis{frigola2015bayesian,
  author = {Frigola, Roger},
  school = {University of Cambridge},
  title  = {Bayesian time series learning with {G}aussian processes},
  year   = {2015},
}

@article{ghosh2024danse,
  title={{DANSE}: Data-driven non-linear state estimation of model-free process in unsupervised learning setup},
  author={Ghosh, Anubhab and Honor{\'e}, Antoine and Chatterjee, Saikat},
  journal   = {IEEE Trans. Signal Process.},
  pages = {1824--1838},
  year={2024},
  publisher={IEEE}
}

@inproceedings{
gu2023mamba,
title={Mamba: Linear-Time Sequence Modeling with Selective State Spaces},
author={Albert Gu and Tri Dao},
booktitle={Proc. First Conference on Language Modeling},
year={2024},
month = oct,
}

@inproceedings{lorenz1996predictability,
  title={Predictability: A problem partly solved},
  author={Lorenz, Edward N},
  booktitle={Proc. Seminar on predictability},
  volume={1},
  number={1},
  pages={1--18},
  year={1996},
  organization={Reading}
}

@article{girin2021dynamical,
  title={Dynamical Variational Autoencoders: A Comprehensive Review},
  author={Girin, Laurent and Leglaive, Simon and Bie, Xiaoyu and Diard, Julien and Hueber, Thomas and Alameda-Pineda, Xavier},
  journal={Found. Trends Mach. Learn.},
  volume={15},
  number={1-2},
  pages={1--175},
  year={2021},
  publisher={Now Publishers Boston—Delft}
}

@article{gneiting2007probabilistic,
  title={Probabilistic forecasts, calibration and sharpness},
  author={Gneiting, Tilmann and Balabdaoui, Fadoua and Raftery, Adrian E},
  journal={J. R. Stat. Soc. Ser. B Stat. Method.},
  volume={69},
  number={2},
  pages={243--268},
  year={2007},
  publisher={Oxford University Press}
}

@article{spantini2022coupling,
  title={Coupling techniques for nonlinear ensemble filtering},
  author={Spantini, Alessio and Baptista, Ricardo and Marzouk, Youssef},
  journal={SIAM Review},
  volume={64},
  number={4},
  pages={921--953},
  year={2022},
  publisher={SIAM}
}

@article{burt2020convergence,
  title={Convergence of sparse variational inference in {G}aussian processes regression},
  author={Burt, David R and Rasmussen, Carl Edward and Van Der Wilk, Mark},
  journal={J. Mach. Learn. Res.},
  volume={21},
  number={131},
  pages={1--63},
  year={2020}
}

@inproceedings{lin2023towards_efficient,
  title={Towards Efficient Modeling and Inference in Multi-Dimensional {G}aussian Process State-Space Models},
  author={Lin, Zhidi and Maro{\~n}as, Juan and Li, Ying and Yin, Feng and Theodoridis, Sergios},
  booktitle = {Proc. IEEE Int. Conf. Acoust. Speech Signal Process. (ICASSP)},
  pages={12881--12885},
  year={2024},
}

@article{tobar2015unsupervised,
  title={Unsupervised state-space modeling using reproducing kernels},
  author={Tobar, Felipe and Djuri{\'c}, Petar M and Mandic, Danilo P},
  journal   = {IEEE Trans. Signal Process.},
  volume={63},
  number={19},
  pages={5210--5221},
  year={2015},
  publisher={IEEE}
}

@article{lin2023towards,
  title={Towards Flexibility and Interpretability of {G}aussian Process State-Space Model},
  author={Lin, Zhid and Yin, Feng and Maro\~{n}as, Juan},
  journal={arXiv preprint arXiv:2301.08843},
  year={2023},
}

@inproceedings{lindinger2022laplace,
  title={Laplace approximated {Gaussian} process state-space models},
  author={Lindinger, Jakob and Rakitsch, Barbara and Lippert, Christoph},
  booktitle={Proc. Conf. Uncertain. Artif. Intell. (UAI)},
  year={2022},
  pages={1199--1209},
  volume = 	 {180},
}

@inproceedings{paszke2019pytorch,
  title={PyTorch: an imperative style, high-performance deep learning library},
  author={Paszke, Adam and Gross, Sam and Massa, Francisco and Lerer, Adam and Bradbury, James and Chanan, Gregory and Killeen, Trevor and Lin, Zeming and Gimelshein, Natalia and Antiga, Luca and others},
  booktitle={Proc. Adv. Neural Inf. Process. Syst. (NeurIPS)},
  pages={8026--8037},
  year={2019}
}

@InProceedings{curi2020structured,
  author    = {Curi, Sebastian and Melchior, Silvan and Berkenkamp, Felix and Krause, Andreas},
  booktitle = {Proc. Learning for Dynamics and Control (L4DC)},
  title     = {Structured Variational Inference in Partially Observable Unstable {G}aussian Process State Space Models},
  year      = {2020},
  pages     = {147--157},
}

@inproceedings{deisenroth2009analytic,
  title={Analytic moment-based {G}aussian process filtering},
  author={Deisenroth, Marc Peter and Huber, Marco F and Hanebeck, Uwe D},
  booktitle = {Proc. Int. Conf. Mach. Learn. (ICML)},
  pages={225--232},
  year={2009}
}

@InProceedings{doerr2018probabilistic,
  author    = {Doerr, Andreas and Daniel, Christian and Schiegg, Martin and Duy, Nguyen-Tuong and Schaal, Stefan and Toussaint, Marc and Sebastian, Trimpe},
  booktitle = {Proc. Int. Conf. Mach. Learn. (ICML)},
  title     = {Probabilistic Recurrent State-Space Models},
  year      = {2018},
  pages     = {1280--1289},
}

@InProceedings{eleftheriadis2017identification,
  author    = {Eleftheriadis, Stefanos and Nicholson, Tom and Deisenroth, Marc Peter and Hensman, James},
  booktitle = {Proc. Adv. Neural Inf. Process. Syst. (NeurIPS)},
  title     = {Identification of {Gaussian} Process State Space Models},
  year      = {2017},
  pages     = {5309--5319},
}

@InProceedings{frigola2014variational,
  author    = {Frigola, Roger and Chen, Yutian and Rasmussen, Carl E},
  booktitle = {Proc. Adv. Neural Inf. Process. Syst. (NeurIPS)},
  title     = {{Variational Gaussian process state-space models}},
  year      = {2014},
  pages     = {3680--3688},
}

@InProceedings{frigola2013bayesian,
  author    = {Frigola, Roger and Lindsten, Fredrik and Sch{\"o}n, Thomas B and Rasmussen, Carl E},
  booktitle = {Proc. Adv. Neural Inf. Process. Syst. (NeurIPS)},
  title     = {{Bayesian inference and learning in Gaussian process state-space models with particle MCMC}},
  year      = {2013},
  pages     = {3156-3164},
}

@InProceedings{hensman2013gaussian,
  author    = {Hensman, James and Fusi, Nicol{\`o} and Lawrence, Neil D},
  booktitle = {Proc. Conf. Uncertain. Artif. Intell. (UAI)},
  title     = {Gaussian processes for Big data},
  year      = {2013},
  pages     = {282--290},
}

@Article{berntorp2021online,
  author    = {Berntorp, Karl},
  journal   = {Automatica},
  title     = {{Online Bayesian inference and learning of Gaussian-process state--space models}},
  year      = {2021},
  pages     = {109613},
  volume    = {129},
  publisher = {Elsevier},
}

@inproceedings{berntorp2023constrained,
  title={Constrained {G}aussian-Process State-Space Models for Online Magnetic-Field Estimation},
  author={Berntorp, Karl and Menner, Marcel},
  booktitle = {Proc. Int. Conf. Inf. Fusion (FUSION)},
  pages={1--7},
  year={2023},
}

@article{khan2008distributing,
  title={Distributing the {K}alman filter for large-scale systems},
  author={Khan, Usman A and Moura, Jos{\'e} MF},
  journal   = {IEEE Trans. Signal Process.},
  volume={56},
  number={10},
  pages={4919--4935},
  year={2008},
  publisher={IEEE}
}

@InProceedings{ialongo2019overcoming,
  author    = {Ialongo, Alessandro Davide and van der Wilk, Mark and Hensman, James and Rasmussen, Carl Edward},
  booktitle = {Proc. Int. Conf. Mach. Learn. (ICML)},
  title     = {{Overcoming mean-field approximations in recurrent Gaussian process models}},
  year      = {2019},
  pages     = {2931--2940},
}

@Article{ko2011learning,
  author    = {Ko, Jonathan and Fox, Dieter},
  journal   = {Auton. Robots},
  title     = {{Learning GP-BayesFilters via Gaussian process latent variable models}},
  year      = {2011},
  number    = {1},
  pages     = {3--23},
  volume    = {30},
  publisher = {Springer},
}

@Article{kobyzev2020normalizing,
  author    = {Kobyzev, Ivan and Prince, Simon JD and Brubaker, Marcus A},
  journal   = {IEEE Trans. Pattern Anal. Mach. Intell.},
  title     = {{Normalizing flows: An introduction and review of current methods}},
  year      = {2020},
  number    = {11},
  pages     = {3964--3979},
  volume    = {43},
  publisher = {IEEE},
}

@InProceedings{krishnan2017structured,
  author    = {Krishnan, Rahul and Shalit, Uri and Sontag, David},
  booktitle = {Proc. AAAI Conf. Artif. Intell. (AAAI)},
  title     = {Structured inference networks for nonlinear state space models},
  year      = {2017},
  pages     = {2101--2109},
}

@InProceedings{maronas2021transforming,
  author    = {Maro{\~n}as, Juan and Hamelijnck, Oliver and Knoblauch, Jeremias and Damoulas, Theodoros},
  booktitle = {Proc. Int. Conf. Artif. Intell. Stat. (AISTATS)},
  title     = {{Transforming Gaussian processes with normalizing flows}},
  year      = {2021},
  pages     = {1081--1089},
}

@article{evensen1994sequential,
  title={{Sequential data assimilation with a nonlinear quasi-geostrophic model using Monte Carlo methods to forecast error statistics}},
  author={Evensen, Geir},
  journal={J. Geophys. Res. Oceans},
  volume={99},
  number={C5},
  pages={10143--10162},
  year={1994},
  publisher={Wiley Online Library}
}

@Article{rios2019compositionally,
  author    = {Rios, Gonzalo and Tobar, Felipe},
  journal   = {Neural Netw.},
  title     = {{Compositionally-warped Gaussian processes}},
  year      = {2019},
  pages     = {235--246},
  volume    = {118},
  month = oct, 
  publisher = {Elsevier},
}

@Book{sarkka2013bayesian,
  title={Bayesian filtering and smoothing},
  author={S{\"a}rkk{\"a}, Simo and Svensson, Lennart},
  volume={17},
  year={2023},
  publisher={Cambridge university press}
}

@Book{theodoridis2020machine,
  author    = {Theodoridis, Sergios},
  publisher = {Academic Press},
  title     = {Machine Learning: {A Bayesian} and Optimization Perspective},
  year      = {2020},
  edition   = {2nd},
}

@InProceedings{titsias2009variational,
  author    = {Titsias, Michalis},
  booktitle = {Proc. Int. Conf. Artif. Intell. Stat. (AISTATS)},
  title     = {Variational Learning of Inducing Variables in Sparse {Gaussian} Processes},
  year      = {2009},
  pages     = {567--574},
}

@Article{wang2007gaussian,
  author    = {Wang, Jack M and Fleet, David J and Hertzmann, Aaron},
  journal   = {IEEE Trans. Pattern Anal. Mach. Intell.},
  title     = {{Gaussian process dynamical models for human motion}},
  year      = {2007},
  number    = {2},
  pages     = {283--298},
  volume    = {30},
  publisher = {IEEE},
}

@Book{williams2006gaussian,
  author    = {Rasmussen, Carl Edward and Williams, Christopher K. I.},
  publisher = {MIT Press},
  title     = {Gaussian Processes for Machine Learning},
  year      = {2006},
}

@InProceedings{xie2020learning,
  author    = {Xie, Ang and Yin, Feng and Ai, Bo and Zhang, Sha and Cui, Shuguang},
  booktitle = {Proc. Int. Conf. Inf. Fusion (FUSION)},
  title     = {{Learning while tracking: A practical system based on variational Gaussian process state-space model and smartphone sensory data}},
  year      = {2020},
  pages     = {1--7},
}

@Article{revach2022kalmannet,
  author    = {Revach, Guy and Shlezinger, Nir and Ni, Xiaoyong and Escoriza, Adria Lopez and Van Sloun, Ruud JG and Eldar, Yonina C},
  journal   = {IEEE Trans. Signal Process.},
  title     = {{KalmanNet: Neural network aided Kalman filtering for partially known dynamics}},
  year      = {2022},
  pages     = {1532--1547},
  volume    = {70},
  publisher = {IEEE},
}

@Article{kingma2019introduction,
  author    = {Kingma, Diederik P and Welling, Max},
  journal   = {Found. Trends Mach. Learn.},
  title     = {An introduction to variational autoencoders},
  year      = {2019},
  number    = {4},
  pages     = {307--392},
  volume    = {12},
  publisher = {Now Publishers, Inc.},
}

@Article{kullberg2021online,
  author    = {Kullberg, Anton and Skog, Isaac and Hendeby, Gustaf},
  journal   = {IEEE Trans. Signal Process.},
  title     = {Online Joint State Inference and Learning of Partially Unknown State-Space Models},
  year      = {2021},
  pages     = {4149--4161},
  volume    = {69},
  publisher = {IEEE},
}

@article{arulkumaran2017deep,
  title={Deep reinforcement learning: A brief survey},
  author={Arulkumaran, Kai and Deisenroth, Marc Peter and Brundage, Miles and Bharath, Anil Anthony},
  journal={IEEE Signal Process. Mag.},
  volume={34},
  number={6},
  pages={26--38},
  year={2017},
  publisher={IEEE}
}

@Article{deisenroth2013gaussian,
  author    = {Deisenroth, Marc Peter and Fox, Dieter and Rasmussen, Carl Edward},
  journal   = {IEEE Trans. Pattern Anal. Mach. Intell.},
  title     = {Gaussian processes for data-efficient learning in robotics and control},
  year      = {2013},
  number    = {2},
  pages     = {408--423},
  volume    = {37},
  publisher = {IEEE},
}

@Article{deisenroth2011robust,
  author    = {Deisenroth, Marc Peter and Turner, Ryan Darby and Huber, Marco F and Hanebeck, Uwe D and Rasmussen, Carl Edward},
  journal   = {IEEE Trans. Autom. Control},
  title     = {{Robust filtering and smoothing with Gaussian processes}},
  year      = {2012},
  number    = {7},
  pages     = {1865--1871},
  volume    = {57},
  publisher = {IEEE},
}

@Article{ko2009gp,
  author    = {Ko, Jonathan and Fox, Dieter},
  journal   = {Auton. Robots},
  title     = {{GP-BayesFilters: Bayesian filtering using Gaussian process prediction and observation models}},
  year      = {2009},
  number    = {1},
  pages     = {75--90},
  volume    = {27},
  publisher = {Springer},
}

@inproceedings{svensson2016computationally,
  title={Computationally efficient {B}ayesian learning of {G}aussian process state space models},
  author={Svensson, Andreas and Solin, Arno and S{\"a}rkk{\"a}, Simo and Sch{\"o}n, Thomas},
  booktitle={Proc. Int. Conf. Artif. Intell. Stat. (AISTATS)},
  pages={213--221},
  year={2016},
}

@Article{svensson2017flexible,
  author    = {Svensson, Andreas and Sch{\"o}n, Thomas B},
  journal   = {Automatica},
  title     = {A flexible state--space model for learning nonlinear dynamical systems},
  year      = {2017},
  pages     = {189--199},
  volume    = {80},
  publisher = {Elsevier},
}

@inproceedings{fan2023free,
  title={Free-Form Variational Inference for {G}aussian Process State-Space Models},
  author={Fan, Xuhui and Bonilla, Edwin V and O’Kane, Terence and Sisson, Scott A},
  booktitle = {Proc. Int. Conf. Mach. Learn. (ICML)},
  pages={9603--9622},
  year={2023},
}

@article{chen2022autodifferentiable,
  title={Autodifferentiable ensemble {K}alman filters},
  author={Chen, Yuming and Sanz-Alonso, Daniel and Willett, Rebecca},
  journal={SIAM J. Math. Data Sci.},
  volume={4},
  number={2},
  pages={801--833},
  year={2022},
  publisher={SIAM}
}

@InProceedings{liu2020gpssm,
  author    = {Liu, Yuhao and Djuri{\'c}, Petar M},
  booktitle = {Proc. European Signal Proces. Conf. (EUSIPCO)},
  title     = {{Gaussian} Process State-Space Models with Time-Varying Parameters and Inducing Points},
  year      = {2021},
  pages     = {1462--1466},
}

@Article{yan2020gaussian,
  author    = {Yan, Zun and Cheng, Peng and Chen, Zhuo and Li, Yonghui and Vucetic, Branka},
  journal   = {IEEE Trans. Signal Process.},
  title     = {Gaussian process reinforcement learning for fast opportunistic spectrum access},
  year      = {2020},
  pages     = {2613--2628},
  volume    = {68},
  publisher = {IEEE},
}

@article{lin2023ensemble,
  title={Ensemble {K}alman Filtering Meets {G}aussian Process {SSM} for Non-Mean-Field and Online Inference},
  author={Lin, Zhidi and Sun, Yiyong and Yin, Feng and Thi{\'e}ry, Alexandre},
  journal={IEEE Trans. Signal Process.},
  volume={72},
  month = aug,
  pages = {4286--4301},
  year={2024}
}

@article{liu2023sequential,
  title={Sequential Estimation of {G}aussian Process-based Deep State-Space Models},
  author={Liu, Yuhao and Ajirak, Marzieh and Djuri{\'c}, Petar M},
  journal={IEEE Trans. Signal Process.},
  volume    = {71},
  pages = {2968--2980},
  year={2023},
  publisher={IEEE}
}

@InProceedings{dinh2017density,
  author    = {Dinh, Laurent and Bengio, Samy},
  booktitle = {Proc. Int. Conf. Learn. Represent. (ICLR)},
  title     = {{Density estimation using Real NVP}},
  year      = {2017},
}

@InProceedings{kingma2015adam,
  author    = {Kingma, Diederik P and Ba, Jimmy},
  booktitle = {Proc. Int. Conf. Learn. Represent. (ICLR)},
  title     = {Adam: A Method for Stochastic Optimization},
  year      = {2015},
}

@InProceedings{lin2022output,
  author  = {Lin, Zhidi and Cheng, Lei and Yin, Feng and Xu, Lexi and Cui, Shuguang},
  booktitle = {Proc. IEEE Int. Conf. Acoust. Speech Signal Process. (ICASSP)},
  title   = {Output-Dependent {G}aussian Process State-Space Model},
  year    = {2023},
  pages     = {1--5},
}

@inproceedings{maronas2022efficient,
  title={Efficient Transformed {G}aussian Processes for Non-Stationary Dependent Multi-class Classification},
  author={Maro{\~n}as, Juan and Hern{\'a}ndez-Lobato, Daniel},
  booktitle={Proc. Int. Conf. Mach. Learn. (ICML)},
  pages={24045--24081},
  year={2023},
}
 \renewcommand{\appendixname}{Supplement}
\appendices
\onecolumn 

\section{Discussions about ELBO Derivation} \label{supp:connenction_to_EnVI}
Following the GPSSM literature \cite{lin2023ensemble,ialongo2019overcoming}, an alternative derivation of the variational inference algorithm presented in this work can begin with the following ELBO, which explicitly defines the variational distribution over the latent states: 
\begin{equation} 
\label{eq:ELBO_general_supp} 
\mathcal{L}_{2} = \mathbb{E}_{q(\x_{0:T}, \bm{f}(\cdot))} \left[\log \frac{p(\y_{1:T}, \x_{0:T}, \bm{f}(\cdot))} {q(\x_{0:T}, \bm{f}(\cdot)) }\right].
\end{equation}
It turns out that $\mathcal{L}_2$ is a lower bound of $\mathcal{L}$ in Eq.~(\textcolor{red}{22}) as summarized in Proposition~\textcolor{red}{1} in the main text.
Next, we show how we can use the methodology proposed in \cite{lin2023ensemble} to derive the same ELBO as in Eq.~(\textcolor{red}{26}). 

The joint distribution of the proposed ETGPSSM is given by:
\begin{equation}
\label{eq:ETGPSSM_joint_supp}
     p( \y_{1:T}, \x_{0:T}, \mathbf{w}, \tilde{f})  =  p(\y_{1:T}, \x_{0:T} \vert \mathbf{w}, \tilde{f}) p_{\bm{\psi}}(\mathbf{w}) p(\tilde{f}).
\end{equation}
The term $p(\y_{1:T}, \x_{0:T} \vert \mathbf{w}, \tilde{f})$ is factorized as:
\begin{equation}
     p(\y_{1:T}, \x_{0:T} \vert \mathbf{w}, \tilde{f})\!=\!p(\x_0)\! \prod_{t=1}^T \underbracket{p(\y_t \vert \x_t)}_{\text{likelihood}} \underbracket{p(\x_t \vert  \mathbf{w}, \tilde{f}, \x_{t-1})}_{\text{transition}},
\end{equation}
where the transition term is specified as:
\begin{align}
    p(\x_t \vert  \mathbf{w}, \tilde{f}, \x_{t-1}) = \int p(\x_t \vert \f_t) p(\f_t \vert \mathbf{w}, \tilde{f}, \x_{t-1}) \mathrm{d} \f_t  = \cN\left(\x_{t} \mid \bm{\alpha}_t \cdot \tilde{f}(\x_{t-1}) + \bm{\beta}_t, \ \mathbf{Q} \right),
\end{align}
with $p(\f_t \vert \mathbf{w}, \tilde{f}, \x_{t-1}) = \delta\left( \f_t - (\bm{\alpha}_t \cdot \tilde{f}(\x_{t-1}) + \bm{\beta}_t) \right)$.

We can employ the following variational distribution to infer the latent variables $\{\x_{0:T}, \mathbf{w}\}$ and the GP $\tilde{f}$: 
\begin{equation} 
\label{eq:ETGPSSM_variational_supp}
q(\x_{0:T}, \tilde{f}, \mathbf{w}) = q(\tilde{f}) q(\mathbf{w}) q(\x_0)\prod_{t=1}^T q(\x_t \vert \mathbf{w}, \tilde{f}, \x_{t-1}), 
\end{equation} 
which mirrors the generative structure of the model in Eq.~\eqref{eq:ETGPSSM_joint_supp} and assumes independence between $\tilde{f}$ and $\mathbf{w}$. 

Given the joint distribution of our model in Eq.~\eqref{eq:ETGPSSM_joint_supp} and the variational distribution in Eq.~\eqref{eq:ETGPSSM_variational_supp}, augmented by the sparse GP, we can formulate the ELBO $\mathcal{L}_2$ according to the general definition provided in Eq.~\eqref{eq:ELBO_general_supp}. The detailed derivations are presented as follows.
\begin{subequations}
\label{eq:ELBO_derivations}
    \begin{align}
        \mathcal{L}_2 & = \mathbb{E}_{q}\left[\log \frac{p(\y_{1:T}, \x_{0:T}, \mathbf{w}, \tilde{f}, \u)}{q(\x_{0:T}, \mathbf{w}, \tilde{f}, \u)} \right] \\
        & = \mathbb{E}_{q}\left[\log \frac{p(\tilde{f},\u) p(\mathbf{w}) p(\x_0) \prod_{t=1}^T p(\y_t \vert \x_t) p(\x_t \vert \mathbf{w}, \tilde{f}, \x_{t-1})}{q(\tilde{f},\u) q(\mathbf{w}) q(\x_0)\prod_{t=1}^T q(\x_t \vert \mathbf{w}, \tilde{f}, \x_{t-1})}\right]\\
        & =  \mathbb{E}_{q}\left[ \sum_{t=1}^T \log \frac{p(\y_t, \x_t \vert \mathbf{w}, \tilde{f}, \x_{t-1})}{q(\x_t \vert \mathbf{w}, \tilde{f}, \x_{t-1}) } \right] -\operatorname{KL}(q(\u) \| p(\u))   -\operatorname{KL}(q(\mathbf{w}) \| p(\mathbf{w})) -\operatorname{KL}(q(\x_0) \| p(\x_0))\\
        & = \mathbb{E}_{q}\left[ \sum_{t=1}^T\log \frac{p(\y_t, \x_t \vert \mathbf{w}, \tilde{f}, \x_{t-1}) p(\x_{t-1} \vert \mathbf{w}, \tilde{f}, \y_{1:t-1})}{q(\x_t \vert \mathbf{w}, \tilde{f}, \x_{t-1})p(\x_{t-1} \vert \mathbf{w}, \tilde{f}, \y_{1:t-1}) } \right] -\operatorname{KL}(q(\u) \| p(\u))   -\operatorname{KL}(q(\mathbf{w}) \| p(\mathbf{w})) -\operatorname{KL}(q(\x_0) \| p(\x_0)) \label{subeq:elbo_1_original}\\
        & \approx \mathbb{E}_{q}\left[ \sum_{t=1}^T \log \frac{p(\y_t, \x_t \vert \mathbf{w}, \tilde{f}, \x_{t-1}) p(\x_{t-1} \vert \mathbf{w}, \tilde{f}, \y_{1:t-1})}{ \underbracket{p(\x_t \vert \mathbf{w}, \tilde{f}, \x_{t-1}, \y_{1:t})}_{\text{assumption 1}} \underbracket{p(\x_{t-1} \vert \mathbf{w}, \tilde{f}, \y_{1:t})}_{\text{assumption 2}} } \right] -\operatorname{KL}(q(\u) \| p(\u))   -\operatorname{KL}(q(\mathbf{w}) \| p(\mathbf{w})) -\operatorname{KL}(q(\x_0) \| p(\x_0)) \label{subeq:elbo_2_approx}\\
        & = \mathbb{E}_{q}\left[ \sum_{t=1}^T  \log \frac{p(\y_t, \x_t, \x_{t-1} \vert \mathbf{w}, \tilde{f}, \y_{1:t-1})}{p(\x_t, \x_{t-1} \vert \mathbf{w}, \tilde{f}, \y_{1:t-1}, \y_t)} \right] -\operatorname{KL}(q(\u) \| p(\u))   -\operatorname{KL}(q(\mathbf{w}) \| p(\mathbf{w})) -\operatorname{KL}(q(\x_0) \| p(\x_0)) \label{subeq:elbo_3}\\
        & =  \mathbb{E}_{q(\mathbf{w}, \tilde{f})}\left[ \sum_{t=1}^T \log p(\y_t \vert \y_{1:t-1}, \mathbf{w}, \tilde{f}) \right] -\operatorname{KL}(q(\u) \| p(\u))   -\operatorname{KL}(q(\mathbf{w}) \| p(\mathbf{w})) -\operatorname{KL}(q(\x_0) \| p(\x_0)), \label{subeq:elbo_final}
    \end{align}
\end{subequations}
which is the objective function we have used in our paper. 
Note that from Eq.~\eqref{subeq:elbo_1_original} to Eq.~\eqref{subeq:elbo_2_approx}, we have made the following two assumptions, which is similar to the work in \cite{lin2023ensemble}:
\begin{assumption}
\begin{align}
  q(\x_{t} \vert \mathbf{w}, \tilde{f}, \x_{t-1}) \approx p(\x_{t} \vert \mathbf{w}, \tilde{f}, \x_{t-1}, \y_{1:t}).
\end{align}
\end{assumption}
\begin{assumption}
\begin{equation}
  p(\x_{t-1} \vert \mathbf{w}, \tilde{f}, \y_{1:t-1}) \approx p(\x_{t-1} \vert \mathbf{w}, \tilde{f}, \y_{1:t}).
\end{equation}
\end{assumption}
\noindent These two approximations used here are similar to those in \cite{lin2023ensemble}, with the key difference being the focus on the ETGP in this paper rather than the GP. The rationale behind these approximations is as follows:
\begin{itemize}
		\item Assumption 1) assumes that the posterior distribution of the latent state at time \( t-1 \) based on observations up to \( t-1 \) is nearly identical to that based on observations up to \( t \). This holds particularly well when \( t \) is large, as the additional observation at time \( t \) has minimal impact on the state estimate due to the accumulated information over time.
		\item Assumption 2) posits that the variational distribution $q(\x_{t} \vert \mathbf{w}, \tilde{f}, \x_{t-1})$ is approximated by the filtering distribution $ p(\x_{t} \vert \mathbf{w}, \tilde{f}, \x_{t-1}, \y_{1:t})$. Note that in variational inference, the variational distribution serves as an approximation of the true smoothing distribution $ p(\x_{t} \vert \mathbf{w}, \tilde{f}, \x_{t-1}, \y_{1:T})$. While there may be some estimation loss, the approximation is reasonable for long observation sequences where future observations minimally affect state estimates. This trade-off balances accuracy and computational efficiency, a crucial consideration in high-dimensional settings.
\end{itemize}

\begin{remark} 
It can be observed that, by adopting the approach from \cite{lin2023ensemble}, which relies on two specific assumptions, one can also derive the same objective function as presented in our main text. However, our contribution in this paper lies in providing an alternative perspective—one that is more straightforward and assumption-free—by leveraging only the EnKF for inference. This not only simplifies the derivation but also enhances the practicality of the method. 
\end{remark}

\section{Implementation Details} \label{supp:implementation_comparison}

\cblue{Throughout the paper, we use the network architecture and hyperparameters specified in Table~\ref{tab:hyperparameters}. For further details, please refer to the code repository available online.}
\begin{table}[h]
\centering
\cblue{
\caption{\cblue{Implementation details for experimental comparisons.}}
\label{tab:hyperparameters}
\begin{tabular}{rccc}
\toprule
\textbf{Component} & \textbf{ETGPSSM} & \textbf{AD-EnKF} \cite{chen2022autodifferentiable} & \textbf{EnVI} \cite{lin2023ensemble} \\
\midrule\midrule
\textbf{Transition Model} & ETGP (1 GP + NN flows) & FCN & $d_x$ independent GPs \\
\textbf{FCN Architecture} & $d_x\to128\to64\to2d_x$ &  $d_x\to128\to64\to d_x$ & -- \\
\textbf{Ensemble Size} ($N$) & 200 & 200 & 200 \\
\textbf{Optimizer} & Adam & Adam & Adam \\
\textbf{Learning Rate} & 0.005 & 0.005 & 0.005 \\
\textbf{Batch Size} & Full sequence & {Full sequence} & Full sequence \\
\textbf{Kernel Function} & SE & -- & SE \\
\textbf{Activation Function} & ReLU & ReLU & -- \\
\textbf{Training Epochs} & 1000 & 1000 & 1000 \\
\textbf{Early Stopping} & Yes & Yes & Yes \\
\bottomrule
\end{tabular}
}
\end{table}

\cblue{
In Section IV-B of the main paper, we adopt the following four performance metrics \cite{spantini2022coupling} to collectively provide a comprehensive assessment of both point estimation accuracy and uncertainty quantification in filtering.  
\begin{itemize}
    \item \textbf{\texttt{Root Mean Squared Error (RMSE)}}: Captures the average deviation between the estimated states and the true reference states.  
    \item \textbf{\texttt{Spread}}: Defined as the root mean trace of the ensemble covariance, this metric reflects the internal variability of the ensemble.
    \item \textbf{\texttt{Coverage Probability}}: Measures how frequently the true states fall within the 95\% empirical confidence intervals (constructed from the 2.5\% and 97.5\% quantiles) of the ensemble marginals.  
    \item \textbf{\texttt{Continuous Ranked Probability Score (CRPS)}}: A comprehensive scoring rule that quantifies how well the predictive distribution matches the observed outcomes \cite{gneiting2007probabilistic}. It evaluates the discrepancy between the forecasted and actual cumulative distribution functions, with lower values signifying more accurate predictions.  
\end{itemize}
}

\end{document}